\documentclass{article}

\usepackage[main, final]{neurips_2026}

\usepackage[utf8]{inputenc} %
\usepackage[T1]{fontenc}    %
\usepackage{hyperref}       %
\usepackage{url}            %
\usepackage{booktabs}       %
\usepackage{amsfonts}       %
\usepackage{nicefrac}       %
\usepackage{microtype}      %
\usepackage{xcolor}         %

\definecolor{darkblue}{rgb}{0, 0, 0.5}
\hypersetup{colorlinks=true, citecolor=darkblue, linkcolor=darkblue, urlcolor=darkblue}

\usepackage{amsmath}
\usepackage{amssymb}
\usepackage{amsthm}
\usepackage{mathrsfs}
\usepackage{times}
\usepackage{latexsym}
\usepackage{amsfonts}
\usepackage{mathrsfs}
\usepackage{mathtools}
\usepackage{mleftright}
\usepackage{xfrac}

\usepackage{inconsolata}

\usepackage{tikz}
\usetikzlibrary{matrix,positioning,calc,decorations.pathreplacing}
\usepackage{xcolor}
\usepackage{url}

\newcommand{\LineComment}[1]{\hfill$\triangleright$~#1}

\theoremstyle{plain}
\newtheorem{example}{Example}

\usepackage[capitalize,noabbrev]{cleveref}
\crefformat{section}{\S#2#1#3}
\Crefformat{section}{\S#2#1#3}

\crefformat{subsection}{\S#2#1#3}
\Crefformat{subsection}{\S#2#1#3}

\crefname{equation}{Eq.}{Eqs.}

\usepackage{thm-restate}
\usepackage[capitalize,noabbrev]{cleveref}

\theoremstyle{plain}
\newtheorem{theorem}{Theorem}[section]

\newtheorem{lemma}[theorem]{Lemma}
\newtheorem{corollary}[theorem]{Corollary}

\newtheorem{definition}[theorem]{Definition}

\theoremstyle{remark}

\newenvironment{proofsketch}{\begin{proof}[Proof sketch]}{\end{proof}}

\usepackage{algorithm}
\usepackage[textsize=tiny]{todonotes}
\newcommand*\iftodonotes{\if@todonotes@disabled\expandafter\@secondoftwo\else\expandafter\@firstoftwo\fi}  %

\newcommand{\note}[4][]{}
\newcommand{\Irene}[2][]{\noindentaftertodo}

\newcommand{\Alexandra}[2][]{\noindentaftertodo}

\renewcommand{\cite}[1]{\citep{#1}}
\newcommand{\citeposs}[1]{\citeauthor{#1}'s~(\citeyear{#1})}

\RequirePackage{algorithm}
\RequirePackage{algorithmic}

\colorlet{MacroColor}{black}

\newcommand{\mymacro}[1]{{\color{MacroColor} #1}}

\newcommand{\defn}[1]{\textbf{#1}}

\newcommand{\paroutline}[3][false]{%
    \ifnum\pdfstrcmp{#1}{true}=0
        #3%
    \else
        [\textit{\textcolor{DiverseMagenta}{#2}}] \textcolor{AccentBlue}{#3}%
    \fi
}

\newcommand{\func}{{\mymacro{ f}}}

\newcommand{\alphabet}{{\mymacro{ \Sigma}}}

\newcommand{\kleene}[1]{{\mymacro{#1^*}}}

\newcommand{\str}{{\mymacro{\boldsymbol{y}}}}

\newcommand{\strx}{{\mymacro{\boldsymbol{x}}}}
\newcommand{\stry}{{\mymacro{\boldsymbol{y}}}}

\newcommand{\syma}{{\mymacro{a}}}
\newcommand{\symb}{{\mymacro{b}}}

\newcommand{\precision}[2]{\mymacro{\pi(#1,#2)}}
\newcommand{\upperprecision}[2]{\mymacro{\pi^\ast(#1,#2)}}
\newcommand{\lowerprecision}[2]{\mymacro{\pi_\ast(#1,#2)}}
\newcommand{\setupperprecision}[2]{\mymacro{\widehat{\pi}^\ast(#1,#2)}}
\newcommand{\setuppertailprecision}[2]{\mymacro{\widehat{\tau}^\ast(#1,#2)}}
\newcommand{\setlowerprecision}[2]{\mymacro{\widehat{\pi}_\ast(#1,#2)}}
\newcommand{\setlowertailprecision}[2]{\mymacro{\widehat{\tau}_\ast(#1,#2)}}
\newcommand{\precisionratio}{\mymacro{p_n}}

\newcommand{\intersect}{\mymacro{\mathcal{I}}}
\newcommand{\Pod}{\mymacro{P}}
\newcommand{\podpool}{\mymacro{\mathcal{P}}}
\newcommand{\chain}{\mymacro{C}}
\newcommand{\unavailableset}{\mymacro{U}}

\newcommand{\recall}[2]{\mymacro{\rho(#1,#2)}}
\newcommand{\upperrecall}[2]{\mymacro{\rho^\ast(#1,#2)}}
\newcommand{\lowerrecall}[2]{\mymacro{\rho_\ast(#1,#2)}}
\newcommand{\setupperrecall}[2]{\mymacro{\widehat{\rho}^\ast(#1,#2)}}
\newcommand{\setlowerrecall}[2]{\mymacro{\widehat{\rho}_\ast(#1,#2)}}

\newcommand{\exploration}{\mymacro{R}}
\newcommand{\explorationidx}{\mymacro{r}}
\newcommand{\langs}{\mymacro{S}}
\newcommand{\targetlang}{\mymacro{T}}
\newcommand{\advlang}{\mymacro{A}}
\newcommand{\guesslang}{\mymacro{G}}

\newcommand{\busybeaver}[1]{\mymacro{\mathrm{BB}(#1)}}

\newcommand{\exhaustion}[1]{\mymacro{\boldsymbol{#1}}}

\newcommand{\exhaustionset}[1]{\mymacro{\mathcal{E}(#1)}}
\newcommand{\exhaustioncompset}[1]{\mymacro{\mathcal{E}_{\mathrm{C}}(#1)}}
\newcommand{\exhaustionsetbounded}[1]{\mymacro{\mathcal{E}_{f}(#1)}}
\newcommand{\exhaustionsetenum}[1]{\mymacro{\mathcal{E}_{e}(#1)}}
\newcommand{\lambdarecall}{\mymacro{\lambda_r}}
\newcommand{\lambdaprecision}{\mymacro{\lambda_p}}
\newcommand{\lambdatail}{\mymacro{\lambda_t}}

\newcommand{\coll}{\mymacro{\sT}}

\newcommand{\gen}{\mymacro{\sG}}

\newcommand{\defeq}{\mathrel{\stackrel{\textnormal{\tiny def}}{=}}}

\newcommand{\bigo}{{\mymacro{ \mathcal{O}}}}

\newcommand{\negterm}[1]{{\mymacro{ {\raise.17ex\hbox{$\scriptstyle\sim$}} #1}}}

\newcommand{\ignore}[1]{}
\newcommand{\expandLater}[1]{}

\def\1{\mathbf{1}}

\def\sG{{{\mymacro{ \mathcal{G}}}}}

\def\sT{{{\mymacro{ \mathcal{T}}}}}

\newcommand{\N}{{\mymacro{ \mathbb{N}}}}

\title{Generating in the Limit with Infinitely Many Hallucinations}

\author{%
  Irene Strauss\thanks{Work done while at ETH Zürich.}\\
 University of Cambridge\\
 \texttt{\href{mailto:is657@cam.ac.uk}{is657@cam.ac.uk}}\\
 \And
 Alexandra Butoi\\
 ETH Zürich \\
 \texttt{\href{mailto:abutoi@ethz.ch}{abutoi@ethz.ch}}\\
 \And
 Ryan Cotterell\\
 ETH Zürich\\
 \texttt{\href{mailto:rcotterell@ethz.ch}{rcotterell@ethz.ch}}\\
}

\begin{document}

\maketitle

\begin{abstract}
The classic paradigm of language identification in the limit models learning as a game between an adversary, who reveals strings from an unknown target language, and a learner tasked with identifying that language. The recently introduced framework of language generation in the limit shifted the objective to better reflect modern language modeling, requiring the learner to produce valid, unseen strings from the target language. Related work highlighted a fundamental tension: a broad coverage of the target often comes at the cost of validity.
We introduce a new notion of precision and recast this problem within a precision--recall framework. We analyze generation in the limit under varying constraints on enumeration, novelty, and validity, aimed at reflecting settings closer to those encountered by large language models. A key contribution is our analysis of learners that are not eventually valid: we allow infinitely many mistakes, provided their frequency tends to zero so that precision remains one. We show that this relaxation can strictly increase recall when the adversary permanently withholds a large portion of the target language. We also study a continuous relaxation of the novelty constraint that requires only a fixed fraction of outputs to be novel. Taken together, our results move toward a more realistic model of language generation where occasional errors and repetitions are unavoidable, but their rates are controlled.\looseness=-1
\end{abstract}

\section{Introduction}

Language identification in the limit is an influential learning paradigm \citep{GOLD1967447,angluin1980inductive}.
In the classic Gold--Angluin model, a learner observes an infinite stream of strings from a target language $\targetlang$, where $\targetlang$ is assumed to belong to a countable collection of languages $\coll$.
At each timestep, the learner outputs a hypothesis in $\coll$; it succeeds if there exists some finite (but unknown) time after which its hypotheses stabilize to the correct language $\targetlang$.
In full generality, identification in this setting is impossible for most interesting collections of languages, i.e., even for regular languages. \citet{angluin1980inductive} gave a characterization of the countable collections $\coll$ that allow identification.
In recent work, \citet{kleinberg2024language} introduced the paradigm of language \emph{generation} in the limit, inspired by modern language models.
Again, the learner is given access to a countable collection of languages $\coll$. Instead of identifying the target language $\targetlang$, it must produce, after some finite timestep, only \emph{valid} (belonging to $\targetlang$), previously unseen strings.
The authors showed that generation in the limit is tractable for any countable collection of languages $\coll$.\looseness=-1

While the algorithm of \citet{kleinberg2024language} ensures validity, it leaves the question of \emph{coverage} unexplored, i.e., how much of $\targetlang$ is generated.
A learner can satisfy the validity requirement by generating from some arbitrarily small, albeit still infinite, subset of $\targetlang$.
In doing so, the learner avoids errors but fails to cover most of $\targetlang$.
Subsequent work formalized this tension, showing a trade-off between coverage and validity \cite{charikar_characterization_2025} and establishing that achieving full coverage is as hard as language identification \cite{kalavasis_characterizations_2025}.
Recently, \citet{kleinberg_density_2025} introduced a notion of \emph{density}. In follow-up work \citep{kleinberg_language_2025}, they prove that, in a partial enumeration setting, i.e., when the learner is shown strings from a subset of the target language $\targetlang$ with density $\alpha$, a language can be generated with density $\sfrac{\alpha}{2}$.

We recast this problem as a recall--precision problem, starting from the observation that the coverage from previous work is equivalent to recall. We introduce a complementary notion of precision that quantifies the learner's error rate, better reflecting modern language-modeling settings where occasional errors are unavoidable and the key question is their rate. Because the standard definitions of precision and recall are ill-defined for infinite sets, we define them via exhaustions, i.e., sequences of finite sets that approximate an infinite set. We also formalize the validity requirement by introducing a notion that we call tail precision.\looseness=-1

We derive precision and recall bounds across a range of settings (summarized in \cref{fig:bounds-summary}), obtained by varying:
(i) the fraction of strings that must be novel,
(ii) whether the adversary reveals a full or only a partial exhaustion of the target language, and
(iii) whether the learner is required to be eventually valid (tail precision of one).
Several of these settings have not been studied previously. 
Our main contributions and improvements over prior work are as follows. First, in the partial enumeration setting with novelty and perfect tail precision, the best known recall bound is due to \citet{kleinberg_language_2025}. We show that relaxing tail precision---allowing hallucinations at a vanishing rate---strictly improves recall whenever the adversary is sufficiently sparse. In the most extreme case, where the adversary reveals almost nothing, our algorithm achieves coverage approaching one, while the no-hallucination bound approaches zero. Second, we introduce a continuous novelty parameter that interpolates between the strict novelty constraint and the no-novelty setting, and give an algorithm with precision one and recall guarantees across the full spectrum. This yields a new setting not studied in prior work: any fixed allowance of non-novel outputs ($\gamma < 1$) allows generation in the limit with recall one and precision one. As technical contributions enabling these results, we introduce an algorithmic framework that interleaves \emph{safe} algorithms with \emph{exploration} rounds that might contain hallucinations, and we generalize the algorithm of \citet{kleinberg_language_2025} to a batched setting where the adversary can reveal multiple strings per round.

\section{Approximate language learning}

\subsection{Preliminaries}

An \defn{alphabet} is a finite, non-empty set $\alphabet$. A \defn{string}, denoted by $\str$ or $\strx$, is a finite sequence of symbols. We write $\N \defeq \{1,2,\ldots\}$ for the positive integers and $\N_0 \defeq \{0,1,2,\ldots\}$ for the non-negative integers.
The Kleene closure of an alphabet $\kleene{\alphabet}$ is the set of all strings over $\alphabet$.
A \defn{language} $\langs$ is a subset of $\kleene{\alphabet}$. We write $(\kleene{\alphabet})^{<\omega}$ for the set of all finite subsets of $\kleene{\alphabet}$.
We call a language \defn{recursively enumerable} if there exists an algorithm (Turing machine) that lists exactly its elements (possibly with repeats), possibly never halting. An \defn{exhaustion} is a sequence of finite sets
$\exhaustion{\langs} \defeq \{\langs_n\}_{n=0}^\infty$ such that
$\langs_0=\emptyset$ and $\langs_n \subseteq \langs_{n+1}$ for all $n$. We define $\Delta \langs_n \defeq \langs_n \setminus \langs_{n-1}$. We write $\langs_n \uparrow \langs \defeq \bigcup_{n=0}^\infty \langs_n$ to denote the limit of an exhaustion. 
When the limit $\langs$ is of relevance, we call $\exhaustion{\langs}$ an exhaustion \emph{of} $\langs$. 
An \defn{enumeration} of a language $\langs$ is an exhaustion $\exhaustion{\langs}=\{\langs_n\}_{n=0}^\infty$ of $\langs$ with the property that $|\Delta \langs_{n}| = 1$ for all $n \in \N$. Given a function $\func\colon \N \to \N_0$ with $\lim_{N\to\infty} \sum_{n=1}^N \func(n) = \infty$, an \defn{$\func$-bounded} exhaustion is an exhaustion $\exhaustion{\langs} = \{\langs_n\}_{n=0}^\infty$ of $\langs$ with the property that $|\Delta \langs_{n}| \leq \func(n)$ for all $n \in \N$. If $\func(n) = 1$, we call the exhaustion a \defn{single-step} exhaustion. If $\func(n) = c$, for a constant $c$, we call the exhaustion a \defn{$c$-step} exhaustion. 
Given a language $\langs$, we write $\exhaustionset{\langs}$ to denote the set of all exhaustions of $\langs$, $\exhaustionsetenum{\langs}$ to denote the set of all enumerations of $\langs$, and $\exhaustionsetbounded{\langs}$ to denote the set of all $\func$-bounded exhaustions of $\langs$.

\begin{figure*}
\centering
\newcommand{\boundcell}[4]{%
  $\begin{aligned}
    #1\\[-1pt]
    #2
  \end{aligned}$\\[6pt]
  {\normalfont\scriptsize\selectfont #3\strut}%
  \\[-1pt]
  {\normalfont\scriptsize\selectfont #4\strut}%
}

\resizebox{0.9\textwidth}{!}{%
\begin{tikzpicture}[
  font=\small,
  dataw/.style={text width=0.19\linewidth}, %
  datah/.style={minimum height=2.6cm, anchor=center},     %
  lab/.style={
    draw=none,
    align=left,
    font=\bfseries,
    text width=0.10\linewidth,
  },
  head/.style={
    draw=none,
    align=center,
    font=\bfseries,
    text width=0.19\linewidth,
  },
  cell/.style={
    dataw,
    datah,
    draw=white,
    line width=1pt,
    inner sep=5pt,
    align=center,
  },
  existing/.style={fill=blue!10},
  ours/.style={fill=green!10},
]

\matrix (M) [matrix of nodes,
  nodes in empty cells,
  ampersand replacement=\amp,
  row sep=0pt,
  column sep=0pt,
  column 1/.style={nodes={lab}},
  row 1/.style={nodes={head}},
  row 1 column 1/.style={nodes={text width=0.1\linewidth}},
  row 2 column 1/.style={nodes={lab, rotate=90, align=center}},
  row 3 column 1/.style={nodes={lab, rotate=90, align=center, text width=2cm}},
]{
  \amp Partial\amp Full\amp Partial \amp Full\\

  Novelty
  \amp |[cell,existing]| {\boundcell
      {\footnotesize\setlowerprecision{\targetlang}{\exhaustion{\guesslang}} &= 1}
      {\footnotesize\setlowerrecall{\exhaustion{\targetlang}}{\guesslang} &\geq \alpha/2}
      {\Cref{thm:half-lower-bound}}
      {\cite{kleinberg_language_2025}}
    }
  \amp |[cell,existing]| {\boundcell
      {\footnotesize\setlowerprecision{\targetlang}{\exhaustion{\guesslang}} &= 1}
      {\footnotesize\setlowerrecall{\exhaustion{\targetlang}}{\guesslang} &\geq 1/2}
      {\Cref{thm:half-lower-bound}}
      {\cite{kleinberg_language_2025}}
    }
  \amp |[cell,ours]| {\boundcell
      {\footnotesize\setlowerprecision{\targetlang}{\exhaustion{\guesslang}} &= 1}
      {\footnotesize\setlowerrecall{\exhaustion{\targetlang}}{\guesslang} &\ge \hspace{-2pt}
  \left.\begin{smallmatrix}
    {\scriptstyle\max(1{-}\beta,}\\
    {\scriptstyle\alpha/3)}
  \end{smallmatrix}\right.}
      {\Cref{thm:noveltybound}}
      {(ours)}
    }
  \amp |[cell,existing]| {\boundcell
      {\footnotesize\setlowerprecision{\targetlang}{\exhaustion{\guesslang}} &= 1}
      {\footnotesize\setlowerrecall{\exhaustion{\targetlang}}{\guesslang} &\geq 1/2}
      {\Cref{thm:half-lower-bound}}
      {\cite{kleinberg_language_2025}}
    } \\

  $\boldsymbol{\gamma}$-Novelty
  \amp |[cell,ours]| {\boundcell
      {\footnotesize\setlowerprecision{\targetlang}{\exhaustion{\guesslang}} &= 1}
      {\footnotesize\setlowerrecall{\exhaustion{\targetlang}}{\guesslang} &\geq \hspace{-1pt}\alpha}
      {\Cref{thm:bounds-alpha-nov-tail-prec}}
      {(ours)}
    }
  \amp |[cell,ours]| {\boundcell
      {\footnotesize\setlowerprecision{\targetlang}{\exhaustion{\guesslang}} &= 1}
      {\footnotesize\setlowerrecall{\exhaustion{\targetlang}}{\guesslang} &= 1}
      {\Cref{corr:bounds-alpha-nov-perf-tail-prec}}
      {(ours)}
    }
  \amp |[cell,ours]| {\boundcell
      {\footnotesize\setlowerprecision{\targetlang}{\exhaustion{\guesslang}} &= 1}
      {\footnotesize\setlowerrecall{\exhaustion{\targetlang}}{\guesslang} &= 1}
      {\Cref{thm:gamma-noveltybound}}
      {(ours)}
    }
  \amp |[cell,ours]| {\boundcell
      {\footnotesize\setlowerprecision{\targetlang}{\exhaustion{\guesslang}} &= 1}
      {\footnotesize\setlowerrecall{\exhaustion{\targetlang}}{\guesslang} &=1}
      {\Cref{corr:full-alpha-nov-zero-tail-prec}}
      {(ours)}
    } \\
};

\draw[-, line width=0.8pt]
  ($(M-1-2.north west)+(0,0.5ex)$) -- ($(M-1-3.north east)+(0,0.5ex)$)
  node[midway, yshift=2ex, font=\bfseries]%
  {Perfect tail precision};

\draw[-, line width=0.8pt]
  ($(M-1-4.north west)+(0,0.5ex)$) -- ($(M-1-5.north east)+(0,0.5ex)$)
  node[midway, yshift=2ex, font=\bfseries]
  {Relaxed tail precision};
\end{tikzpicture}%
}

\caption{Overview of precision ($\widehat{\pi}_\ast$) and recall ($\widehat{\rho}_\ast$) bounds across settings defined by three constraints: novelty fraction with $\gamma\in [0,1)$ vs. novelty, partial vs. full adversarial exhaustion, and perfect tail precision ($\setlowertailprecision{\targetlang}{\exhaustion{\guesslang}} = 1$) vs. relaxed tail precision ($\setlowertailprecision{\targetlang}{\exhaustion{\guesslang}}$ can be $0$). Existing results are in blue and new results in green. $\targetlang$ is the target language and $\exhaustion{\targetlang}$ is an enumeration of $\targetlang$. The adversary reveals a language $\advlang \subseteq \targetlang$ via an exhaustion $\exhaustion{\advlang}$. The learner outputs a guess language $\guesslang$ via the exhaustion $\exhaustion{\guesslang}$. In the full exhaustion setting, the adversary is restricted to revealing full exhaustions; in the partial setting, it may reveal a partial exhaustion. $\exhaustion{\advlang}$ is characterized by lower and upper recall bounds $\setlowerrecall{\exhaustion{\targetlang}}{\advlang} \geq \alpha$ and $\setupperrecall{\exhaustion{\targetlang}}{\advlang} \leq \beta$.
}
\label{fig:bounds-summary}
\end{figure*}

\subsection{Approximate language learning}
In this paper, we are interested in understanding how close one language is to another. 
However, languages are, in general, \emph{infinite} sets.
Consequently, cardinalities do not distinguish between small and large numbers of disagreements. Moreover, since $\kleene{\alphabet}$ is infinite, it admits no uniform distribution. Thus, the error rate $\Pr_{\stry \sim D}[\stry \in \langs \setminus \langs']$, for two languages $\langs, \langs' \subseteq \kleene{\alphabet}$, is not well defined without choosing a particular distribution $D$ over $\kleene{\alphabet}$.\looseness=-1
 
\paragraph{A warm-up.} We first consider \emph{finite} languages as a special case.
Let $\targetlang\subset\kleene{\alphabet}$ be a finite \underline{t}arget language and let $\guesslang\subset \kleene{\alphabet}$ be a finite \underline{g}uess language. Under the standard definitions \citep{manning2008introduction}, the \defn{precision} ($\pi$) of $\guesslang$
with respect to $\targetlang$ is the fraction of predicted strings that are correct, and the
\defn{recall} ($\rho$) is the fraction of ground truth positives (strings in $\targetlang$) that are recovered, 
\vspace{-7pt}
\par\noindent
\begin{minipage}[t]{0.48\linewidth}
\begin{equation}\label{eq:fin-precision}
\precision{\targetlang}{\guesslang} \defeq \frac{|\targetlang \cap \guesslang|}{|\guesslang|},
\end{equation}
\end{minipage}\hfill
\begin{minipage}[t]{0.48\linewidth}
\begin{equation}\label{eq:fin-recall}
\recall{\targetlang}{\guesslang} \defeq \frac{|\targetlang \cap \guesslang|}{|\targetlang|}.
\end{equation}
\end{minipage}
\par\smallskip
The ratios used in precision and recall are not directly applicable to infinite sets since $\frac{\infty}{\infty}$ is not well-defined in our setting.
Hence, we will make use of exhaustions to evaluate performance on increasing finite sets.
Intuitively, an exhaustion of $\guesslang$ models a learner's successive predictions, while an exhaustion of $\targetlang$ corresponds to the finite subsets of the target language on which we assess the learner. For the rest of the paper, we assume $\guesslang$ and $\targetlang$ to be infinite.\looseness=-1

\paragraph{Exhaustion-level precision.}
Let $\exhaustion{\targetlang} = \{\targetlang_n\}_{n=0}^\infty $ be an exhaustion of the target language $\targetlang$, called a \defn{target exhaustion}, and let $\exhaustion{\guesslang} = \{\guesslang_n\}_{n=0}^\infty$ be an exhaustion of the guess language $\guesslang$, called \defn{guess exhaustion}. 
Then, define the \defn{upper precision} $\pi^\ast$ and the \defn{lower precision} $\pi_\ast$ as\looseness=-1
\vspace{-7pt}
\par\noindent
\begin{minipage}[t]{0.48\linewidth}
\begin{equation}
\upperprecision{\exhaustion{\targetlang}}{\exhaustion{\guesslang}}
\defeq \limsup_{n\to\infty} \frac{|\targetlang_n \cap \guesslang_n|}{|\guesslang_n|},
\end{equation}
\end{minipage}\hfill
\begin{minipage}[t]{0.48\linewidth}
\begin{equation}
\lowerprecision{\exhaustion{\targetlang}}{\exhaustion{\guesslang}}
\defeq  \liminf_{n\to\infty} \frac{|\targetlang_n \cap \guesslang_n|}{|\guesslang_n|}.
\end{equation}
\end{minipage}
\par\smallskip
When the upper and lower precision coincide, i.e.,
$\lowerprecision{\exhaustion{\targetlang}}{\exhaustion{\guesslang}}=\upperprecision{\exhaustion{\targetlang}}{\exhaustion{\guesslang}}$, we define the \defn{precision} simply as $\precision{\exhaustion{\targetlang}}{\exhaustion{\guesslang}}
\defeq \lim_{n\to\infty} \frac{|\targetlang_n \cap \guesslang_n|}{|\guesslang_n|}$.
We show that precision does not always exist in \cref{ex:precision-no-limit}.

\paragraph{Exhaustion-level recall.}
Analogously to precision, we define the \defn{upper recall} $\rho^\ast$ and \defn{lower recall} $\rho_\ast$ as\looseness=-1
\vspace{-7pt}
\par\noindent
\begin{minipage}[t]{0.48\linewidth}
\begin{equation}
\upperrecall{\exhaustion{\targetlang}}{\exhaustion{\guesslang}}
\defeq \limsup_{n\to\infty}\frac{|\targetlang_n\cap \guesslang_n|}{|\targetlang_n|},
\end{equation}
\end{minipage}\hfill
\begin{minipage}[t]{0.48\linewidth}
\begin{equation}
\lowerrecall{\exhaustion{\targetlang}}{\exhaustion{\guesslang}}
\defeq \liminf_{n\to\infty}\frac{|\targetlang_n\cap \guesslang_n|}{|\targetlang_n|}.
\end{equation}
\end{minipage}
\par\smallskip
When these coincide, we define the \defn{recall} as $\recall{\exhaustion{\targetlang}}{\exhaustion{\guesslang}}
\defeq \lim_{n\to\infty}\frac{|\targetlang_n\cap \guesslang_n|}{|\targetlang_n|}$.
Just like precision, recall does not always exist (see \Cref{ex:recall-no-limit}).

\paragraph{Membership-based precision.} The precision notions above compare two \emph{exhaustions}. Validity, however, depends on membership in $\targetlang$ rather than membership in a particular finite prefix $\targetlang_n$ of some exhaustion. To eliminate this arbitrary dependence on how $\targetlang$ is exhausted, we define precision based on membership in the target language by taking
the supremum over all exhaustions of $\targetlang$. Since the limit need not exist, we use the lower and upper variants:
\vspace{-7pt}
\par\noindent
\begin{minipage}[t]{0.48\linewidth}
\begin{equation}
\setlowerprecision{\targetlang}{\exhaustion{\guesslang}}
\defeq
\sup_{\exhaustion{\targetlang}\in\exhaustionset{\targetlang}}
\lowerprecision{\exhaustion{\targetlang}}{\exhaustion{\guesslang}},
\end{equation}
\end{minipage}\hfill
\begin{minipage}[t]{0.48\linewidth}
\begin{equation}
\setupperprecision{\targetlang}{\exhaustion{\guesslang}}
\defeq
\sup_{\exhaustion{\targetlang}\in\exhaustionset{\targetlang}}
\upperprecision{\exhaustion{\targetlang}}{\exhaustion{\guesslang}}.
\end{equation}
\end{minipage}
\par\smallskip
We primarily work with $\setlowerprecision{\targetlang}{\exhaustion{\guesslang}}$ as a
robust guarantee; all statements have direct analogues for the upper precision.
As we show next, this definition coincides with the intuitive membership-based notion: it equals the lower
asymptotic fraction of guess strings that lie in $\targetlang$.
\begin{restatable}{theorem}{TheoremLowerPrecisionExhaustion}\label{th:lower-precision-bounded-exhaustions} Let $\func(n) \colon \N \to \N$ be a function with $\lim_{N\to \infty}\sum_{n=1}^N \func(n) = \infty$. Let $\targetlang\subseteq \kleene{\alphabet}$ be the target language and let $\exhaustion{\guesslang} = \{\guesslang_n\}_{n=0}^\infty$ be an $\func$-bounded exhaustion of the guess $\guesslang$. Then, 
\begin{equation}\label{eq:best-lower-precision-exhaustion}
\sup_{\exhaustion{\targetlang}\in\exhaustionsetbounded{\targetlang}} \lowerprecision{\exhaustion{\targetlang}}{\exhaustion{\guesslang}} = \liminf_{n\to\infty}\frac{|\targetlang\cap \guesslang_n|}{|\guesslang_n|}.
\end{equation}
\end{restatable}
The proof is given in \cref{app:membership-precision}.

\paragraph{Coverage-based recall.}
For recall, we are interested in understanding how many strings in the target language $\targetlang$ the learner actually recovers. 
Hence, we check for each string in $\targetlang$ if it is included in $\guesslang$. This is achieved by taking the supremum over $\guesslang$. 
We define
\vspace{-7pt}
\par\noindent
\begin{minipage}[t]{0.48\linewidth}
\begin{equation}\label{eq:recall1}
\setlowerrecall{\exhaustion{\targetlang}}{\guesslang} = \sup_{\exhaustion{\guesslang}\in\exhaustionset{\guesslang}} \lowerrecall{\exhaustion{\targetlang}}{\exhaustion{\guesslang}},
\end{equation}
\end{minipage}\hfill
\begin{minipage}[t]{0.48\linewidth}
\begin{equation}\label{eq:recall2}
\setupperrecall{\exhaustion{\targetlang}}{\guesslang} = \sup_{\exhaustion{\guesslang}\in\exhaustionset{\guesslang}} \upperrecall{\exhaustion{\targetlang}}{\exhaustion{\guesslang}}.
\end{equation}
\end{minipage}
\par\smallskip
This definition assumes that we are given a fixed target exhaustion $\exhaustion{\targetlang}$. Again, we are mostly concerned with the lower recall as a robust guarantee; the following lemma holds analogously for upper recall.
\begin{restatable}{lemma}{LemmaLowerRecall}\label{lem:lower-recall}
Let $\targetlang\subseteq \kleene{\alphabet}$ be the target language with an exhaustion $\exhaustion{\targetlang} = \{\targetlang_n\}_{n=0}^\infty$ and let $\guesslang$ be the guess language. Then, 
\begin{equation}\label{eq:best-lower-recall}
\sup_{\exhaustion{\guesslang}\in\exhaustionset{\guesslang}} \lowerrecall{\exhaustion{\targetlang}}{\exhaustion{\guesslang}} = \liminf_{n\to\infty}\frac{|\targetlang_n\cap \guesslang|}{|\targetlang_n|}.
\end{equation}
\end{restatable}
The proof is given in \cref{app:coverage-recall}.\looseness=-1

\paragraph{Tail precision.}
The precision notions above are global: they measure the fraction of correct strings among
\emph{all} strings in $\guesslang_n$. Now we define \defn{tail precision}, which considers only the increment $\Delta \guesslang_n$. 

Let $\exhaustion{\guesslang}=\{\guesslang_n\}_{n= 0}^\infty$ be an exhaustion.
For each $n$, define the step-wise tail precision $t_n$ as
\begin{equation}
t_n \defeq 
\begin{cases}
1, & \text{if } |\Delta \guesslang_n|=0,\\
\frac{|\targetlang\cap \Delta\guesslang_n|}{|\Delta\guesslang_n|}, & \text{otherwise.}
\end{cases}
\end{equation}

\begin{definition}
    Given a target language $\targetlang$ and a guess language $\guesslang$ with exhaustion $\exhaustion{\guesslang}$, we define the \defn{lower tail precision} and \defn{upper tail precision} as
    \vspace{-7pt}
    \par\noindent
\begin{minipage}[t]{0.48\linewidth}
\begin{equation}
\setlowertailprecision{\targetlang}{\exhaustion{\guesslang}}
\defeq \liminf_{n\to\infty} t_n,
\end{equation}
\end{minipage}\hfill
\begin{minipage}[t]{0.48\linewidth}
\begin{equation}
\setuppertailprecision{\targetlang}{\exhaustion{\guesslang}}
\defeq \limsup_{n\to\infty} t_n.
\end{equation}
\end{minipage}
\end{definition}
Tail precision measures the asymptotic fraction of newly added strings that are in $\targetlang$.

\begin{restatable}{proposition}{PropositionConvergenceTailPrecision}\label{prop:tail-prec-finite-time-conv}
    Let $\targetlang$ be the target language and $\exhaustion{\guesslang}$ a single-step exhaustion of $\guesslang$. Then, $\setlowertailprecision{\targetlang}{\exhaustion{\guesslang}}$ is either $0$ or $1$. Moreover, for any $c$-step exhaustion $\exhaustion{\guesslang}$, $\setlowertailprecision{\targetlang}{\exhaustion{\guesslang}}$ is attained after finitely many steps.\footnote{This means that there exists a timestep $n < \infty$, such that the step-wise tail precision $t_n = \setlowertailprecision{\targetlang}{\exhaustion{\guesslang}}$ and $t_N \geq t_n$ for $N\geq n$.}
\end{restatable}
We give a proof in \cref{app:tail-precision}. For a function $\func$ with non-constant growth, one can distinguish between finite-time stabilization to some constant and asymptotic convergence; the following example illustrates this. \begin{example} Let $|\Delta \guesslang_n|\le n$, i.e., at step $n$ the exhaustion $\exhaustion{\guesslang}$ may grow by up to $n$ new strings. Suppose that at every step the increment contains exactly $n-1$ strings from $\targetlang$ and one string not in $\targetlang$. Then, \begin{equation} \setlowertailprecision{\targetlang}{\exhaustion{\guesslang}} = \liminf_{n\to\infty} \frac{|\targetlang\cap \Delta\guesslang_n|}{|\Delta\guesslang_n|} = \liminf_{n\to\infty} \frac{n-1}{n} = 1, \end{equation} i.e., the lower tail precision converges asymptotically to $1$, even though an error occurs at every step. However, this ratio is never exactly $1$ at any finite step. Finite-time
stabilization to $1$ would require that, from some step onward, every newly
generated string lies in $\targetlang$, i.e.,
$\Delta\guesslang_n\subseteq\targetlang$ for all sufficiently large $n$.\end{example}

\paragraph{Relation to \citet{kleinberg_density_2025}.} Our notion of coverage-based recall is nearly identical to the density measure of \citet{kleinberg_density_2025}.
In their setup, density is evaluated along a fixed enumeration $\exhaustion{\targetlang}=\{\targetlang_n\}_{n=0}^\infty$. In \cref{lem:lower-recall}, we derive \citeposs{kleinberg_density_2025} density measure of one language $\guesslang$ in another $\targetlang$, corresponding to our definitions of lower (\Cref{eq:recall1}) and upper (\Cref{eq:recall2}) coverage-based recall.
We extend their framework in two ways: (i) we allow \emph{any} exhaustion rather than committing
to a single fixed enumeration, and (ii) we introduce a complementary
\emph{precision} notion by exhausting the guessed language instead of the target. In the next section, we instantiate these notions in the generation-in-the-limit framework of \citet{kleinberg2024language}. In that setting, a generator's output stream induces increasing prefix sets $\guesslang_n$ of generated strings, and hence naturally defines an exhaustion of $\guesslang$. We will also formalize the eventual-validity requirement from prior work by showing in \Cref{prop:validity-tail-precision} that it corresponds to \emph{finite-time stabilization} of tail precision to $1$.
\section{Generation in the limit}\label{sec:lgitl}
We now introduce \citeposs{kleinberg2024language} generation in the limit framework. Let $\coll$ be a countable collection of languages $\coll \subseteq \mathcal{P}(\kleene{\alphabet})$. We assume that each language $\langs \in \coll$ is infinite.\footnote{This is in line with \citeauthor{kleinberg2024language}'s work and makes the generation of infinitely many unseen strings possible.} 

An adversary chooses a target language $\targetlang\in\coll$ and produces an $\func$-bounded exhaustion $\exhaustion{\advlang}=\{\advlang_n\}_{n= 0}^\infty$ such that $\advlang_n\subseteq \targetlang$  for all $n$, and $\advlang_n \uparrow \advlang$ with $\advlang \subseteq \targetlang$. \citet{kleinberg2024language} assume a single-step exhaustion; we extend their framework to $\func$-bounded exhaustions. The adversary can repeat strings, and is not required to reveal a new string at each timestep; therefore, $|\Delta \advlang_n|$ can be $0$. We call $\exhaustion{\advlang}$ a \defn{full} exhaustion of $\targetlang$ if $\advlang=\targetlang$, and a \defn{partial} exhaustion otherwise ($\advlang\subsetneq \targetlang$). In line with \citet{kleinberg_language_2025}, we assume $\advlang$ to be infinite.
A \defn{generator} is a function $\gen \colon (\kleene{\alphabet})^{<\omega}\times(\kleene{\alphabet})^{<\omega}\to(\kleene{\alphabet})^{<\omega}$.
Given an adversarial exhaustion  $\exhaustion{\advlang}=\{\advlang_n\}_{n=0}^\infty$,
the generator induces an $\func$-bounded \defn{generated exhaustion} $\exhaustion{\guesslang}=\{\guesslang_n\}_{n=0}^\infty$ via 
\begin{equation}\label{eq:guess-exhaustion}
\guesslang_0 \defeq \emptyset,
\quad
\guesslang_n \defeq \guesslang_{n-1}\cup \gen(\advlang_n,\guesslang_{n-1})\quad (n\in \N).
\end{equation}
At time $n$ the generator observes the currently revealed set $\advlang_n$ and the previously generated set $\guesslang_{n-1}$,
and appends one or multiple new strings. \citet{kleinberg2024language} demand that one new string is generated that is not in $\advlang_n$. Furthermore, their proposed algorithms also do not repeat strings in $\guesslang_{n-1}$.
We write $\guesslang_n \uparrow \guesslang$ for the guess language.

\begin{definition}[Generation in the limit]\label{def:lgitl}
A generator $\gen$ \defn{generates in the limit} with precision $\lambdaprecision$, recall $\lambdarecall$ and tail precision $\lambdatail$ for a collection of languages $\coll$ if for every $\targetlang\in\coll$, every adversarial exhaustion $\exhaustion{\advlang} = \{\advlang_n\}_{n=0}^\infty$ of $\advlang \subseteq \targetlang$, and  $\exhaustion{\guesslang}$ defined as in~\eqref{eq:guess-exhaustion}, it holds that
\vspace{-10pt}
\par\noindent
\begin{subequations}\label{eq:lgitl-conditions}
\begin{minipage}[t]{0.32\linewidth}
\begin{equation}\label{eq:lgitl-precision}
\setlowerprecision{\targetlang}{\exhaustion{\guesslang}}
\geq \lambdaprecision,
\end{equation}
\end{minipage}\hfill
\begin{minipage}[t]{0.32\linewidth}
\begin{equation}\label{eq:lgitl-recall}
\setlowerrecall{\exhaustion{\targetlang}}{\guesslang}
\geq \lambdarecall,
\end{equation}
\end{minipage}\hfill
\begin{minipage}[t]{0.32\linewidth}
\begin{equation}\label{eq:lgitl-novel-output}
\setlowertailprecision{\targetlang}{\exhaustion{\guesslang}}
\geq \lambdatail.
\end{equation}
\end{minipage}
\end{subequations}
\end{definition}
\subsection{Perfect tail precision}

\citet{kleinberg2024language} require that after some finite timestep, all generated strings are valid (i.e., belong to $\targetlang$). 
We refer to this property as \emph{eventual validity}.

\begin{definition}[Eventual Validity]
\label{def:eventual-validity}
A generator $\gen$ is \defn{eventually valid} if there exists an $n^\star$ such that for all $n\ge n^\star$,
\begin{equation}
\gen(\advlang_n,\guesslang_{n-1}) \subseteq \targetlang.
\end{equation}
\end{definition} 
Equivalently, $|\guesslang\setminus\targetlang| < \infty$. We show that eventual validity corresponds to finite-time stabilization of tail precision to $1$.

\begin{restatable}{proposition}{PropositionValidityTailPrecision}\label{prop:validity-tail-precision}
Let $\gen$ be a generator and let $\exhaustion{\guesslang}$ be the exhaustion it produces against the adversarial exhaustion $\exhaustion{\advlang}$ of $\advlang \subseteq \targetlang$.
Then $\gen$ is eventually valid (Def.~\ref{def:eventual-validity}) if and only if
$\setlowertailprecision{\targetlang}{\exhaustion{\guesslang}} = 1$ is attained in finite time.
\end{restatable}
\begin{proofsketch}
    Intuitively, if the step-wise tail precision is $1$ from some finite time
onward, then every newly generated string after that time lies in the target
language $\targetlang$. Hence, only finitely many mistakes can occur, and the
generator is eventually valid. Conversely, if only finitely many mistakes occur,
then after the last mistake every newly generated string is valid, so the
step-wise tail precision is $1$ from that point onward.
\end{proofsketch}
The full proof is in \Cref{app:tail-precision}.  In general, tail precision $1$ automatically implies precision $1$ (see \Cref{lem:tail-precision-implies-precision}). On the other hand, one may have precision $1$ without convergence of tail precision to $1$: for example, if the generator outputs exactly one invalid string at times $n=2^m$ for $m \in \N$ and a valid string otherwise,
then the number of mistakes up to time $n$ is $\Theta(\log n)$, so precision tends to $1$ even though mistakes occur infinitely often, which yields lower tail precision $0$.
\subsection{Generation with novelty}
\begin{definition}[Novelty constraint] 
\label{def:novelty} A generator $\gen$ is \defn{novel} if for every $n\ge 1$,
\begin{equation}
    \gen(\advlang_n,\guesslang_{n-1})
\cap (\advlang_n\cup \guesslang_{n-1})=\emptyset.
\end{equation}
\end{definition}
In words, once a string has been seen---meaning that it has been either revealed in some $\advlang_n$  or it has been generated by some $\guesslang_n$---it becomes permanently unavailable to the generator. This modeling choice makes achieving high recall non-trivial. Without it, a generator could simply copy the adversary and achieve perfect recall and perfect precision whenever the adversary fully enumerates $\targetlang$ (see \Cref{thm:gen-valid-without-novelty}). Under a novel generator, however, the adversary can \emph{steal} strings by revealing them first, which imposes an inherent ceiling on recall. We restate, in our notation, the following recall bounds from \citet{kleinberg_language_2025}.

\begin{restatable}{theorem}{TheoremHalfDensityLower}[\citealt[Thm.~3.5]{kleinberg_language_2025}]
\label{thm:half-lower-bound}
Let $\coll$ be a countable collection of languages and $\targetlang\in\coll$ a target language.
Fix an enumeration $\exhaustion{\targetlang}=\{\targetlang_n\}_{n=0}^\infty$ of $\targetlang$.
There exists a novel generator $\gen$ that generates in the limit for $\targetlang$ with $\setlowertailprecision{\targetlang}{\exhaustion{\guesslang}} = 1$, such that the following holds: For every single-step adversarial exhaustion $\exhaustion{\advlang}=\{\advlang_n\}_{n=0}^\infty$ with $\advlang_n\uparrow \advlang\subseteq \targetlang$, the generator
$\gen$ induces a single-step exhaustion $\exhaustion{\guesslang}=\{\guesslang_n\}_{n= 0}^\infty$ with $\guesslang_n\uparrow \guesslang$ satisfying
\begin{equation}\label{eq:half-density:recall}
\setlowerrecall{\exhaustion{\targetlang}}{\guesslang}
\ge
\frac{1}{2}\,\setlowerrecall{\exhaustion{\targetlang}}{\advlang}.
\end{equation}
\end{restatable}

By \Cref{prop:tail-prec-finite-time-conv}, $\setlowertailprecision{\targetlang}{\exhaustion{\guesslang}}$ converges in finite time to $1$ in this setting. As a special case, \Cref{thm:half-lower-bound} implies that for a full, single-step exhaustion of the target language, an algorithm can generate with recall $\tfrac{1}{2}$.

\begin{restatable}{theorem}{TheoremHalfDensityUpper}[\citealt[Thm.~3.3]{kleinberg_language_2025}]
\label{thm:half-tight}
The constant $\tfrac12$ in \Cref{thm:half-lower-bound} is optimal:
for every novel generator $\gen$ there exist a countable collection $\coll$, a target language
$\targetlang\in\coll$ with enumeration $\exhaustion{\targetlang}$, and a single-step adversarial
exhaustion $\exhaustion{\advlang}=\{\advlang_n\}_{n=0}^\infty$ with $\advlang_n\uparrow \advlang\subseteq \targetlang$
such that, if $\gen$ induces a single-step exhaustion $\exhaustion{\guesslang}$, then
\begin{equation}
\setlowerrecall{\exhaustion{\targetlang}}{\guesslang}\le
\frac12\,\setlowerrecall{\exhaustion{\targetlang}}{\advlang}.
\end{equation}
\end{restatable}

In practice, forbidding the generator from outputting previously revealed strings is not always desirable:
repeating training data can be necessary (e.g., when quoting a source). Motivated by this observation, we introduce the following continuous relaxation of the novelty constraint.

\begin{definition}[$\gamma$-novelty]\label{def:gamma-novelty}
Let $\gamma\in[0,1]$. Given an adversarial exhaustion
$\exhaustion{\advlang}=\{\advlang_n\}_{n=0}^\infty$ and the generated exhaustion $\exhaustion{\guesslang}=\{\guesslang_n\}_{n=0}^\infty$, define
\begin{equation}
N_0 := \emptyset,
\qquad
N_n := N_{n-1}\cup
\bigl(\Delta\guesslang_n\setminus\advlang_n\bigr).
\end{equation}
We say that the generator is $\gamma$-novel if, for every $n$ with
$|\guesslang_n|>0$,
\begin{equation}
\frac{|N_n|}{|\guesslang_n|}\ge \gamma.
\end{equation}
\end{definition}
In words, $\gamma$-novelty requires that every $\guesslang_n$ contains at least a $\gamma$-fraction of strings that
were novel at the time they were first generated, i.e., strings that had not
already been revealed by the adversary and had not already been generated.
Equivalently, at most a $(1-\gamma)$-fraction of the strings in any $\guesslang_n$ may be non-novel. We define $\gamma$-novelty for $\gamma\in[0,1]$. For $\gamma=0$, this imposes no novelty restriction, while the endpoint $\gamma=1$ coincides with the strict novelty constraint of \Cref{def:novelty}. We treat novelty ($\gamma = 1$) separately from the relaxed regime $\gamma<1$ (see \Cref{fig:bounds-summary}).

We study $\gamma$-novel generation in the limit under varying tail-precision assumptions in \Cref{sec:gen-with-infinite-mistakes}. We show that the relaxation $\gamma < 1$ is already sufficient to recover all adversarially revealed strings, and, when perfect tail precision is not required, to achieve both precision and recall equal to $1$.

\section{Generating with infinitely many hallucinations}\label{sec:gen-with-infinite-mistakes}
Existing results aim to maximize recall for novel generators with perfect tail precision, yielding guarantees such as the $\tfrac12$ bound of \citet{kleinberg_language_2025} (restated in \Cref{thm:half-lower-bound,thm:half-tight}).
In this section, we ask how these limits change once we relax the tail precision constraint: rather than requiring that only finitely many hallucinations occur, we allow infinitely many hallucinations provided they occur at a vanishing rate. In this regime, tail precision can even be $0$, but precision is $1$ in the limit.

In line with \citet{kleinberg_language_2025}, we assume a fixed enumeration $\exhaustion{\targetlang}=\{\targetlang_n\}_{n= 0}^\infty$ and single-step exhaustions $\exhaustion{\advlang}=\{\advlang_n\}_{n=0}^\infty$ with $\advlang_n\uparrow\advlang\subseteq\targetlang$ and $\exhaustion{\guesslang}=\{\guesslang_n\}_{n=0}^\infty$ with $\guesslang_n\uparrow\guesslang$, i.e., both the adversary and the generator are allowed to produce at most one string per round. In contrast, an enumeration adds exactly one new string at every round. Our results for generation without novelty ($\gamma = 0$) hold for a more general setting where both $\exhaustion{\advlang}$ and $\exhaustion{\guesslang}$ are $\func$-bounded. We characterize the adversary's exhaustion by the lower and upper recall:
\begin{equation}
    \setlowerrecall{\exhaustion{\targetlang}}{\advlang} \geq \alpha, \qquad \setupperrecall{\exhaustion{\targetlang}}{\advlang} \leq \beta.
\end{equation}
Intuitively, the lower recall measures how much of the target language the adversary eventually reveals, while the upper recall captures how large the revealed set can appear along subsequences. We fix an enumeration of $\Sigma^*$ and let $\prec$ denote the induced order. Throughout this section, we assume that
the fixed target enumeration $\exhaustion{\targetlang}$
is induced by this order, i.e., $\targetlang_n$ consists
of the $n$ smallest elements of $\targetlang$ under $\prec$.

We first analyze this setting without and with the novelty constraint ($\gamma = 0$ vs. $\gamma = 1$), and then give a $\gamma$-novel algorithm and guarantees for $\gamma \in [0, 1)$. We propose an algorithmic template that most of our constructions share: We run a \defn{safe} generator (outputting from $\targetlang$) most of the time, and interleave it with rare \textbf{exploration} rounds that enumerate $\kleene{\alphabet}$ to recover strings the adversary never reveals. Exploration rounds may output invalid strings, but we schedule them sparsely, yielding precision~$1$. To achieve this, we will make use of \defn{exploration sets}. 
\begin{definition}[Exploration set]\label{def:exploration-set}
An \defn{exploration set} is an infinite set $\exploration \subset \N$ such that $1\notin \exploration$ and no two consecutive integers
belong to $\exploration$. For $m\in\N$, define
\(
\explorationidx(m)\defeq \bigl|\exploration\cap\{1,\dots,m\}\bigr|.
\)
We require that $\explorationidx(m)/m \to 0$ as $m\to\infty$.
\end{definition}

An example is $\exploration=\{2^i \colon i\in\N\}$.

For the $\gamma$-novel algorithm, we need exploration rounds to be sparse not only asymptotically, but in every finite prefix. We therefore introduce the following strengthened notion.

\begin{definition}[$\gamma$-admissible exploration set]
\label{def:gamma-admissible-exploration}
Let $\gamma\in[0,1)$.
An infinite exploration set $\exploration\subseteq\N$ is
\defn{$\gamma$-admissible} if, for every $m\in\N$, $\explorationidx(m) \leq (1-\gamma) m$.
\end{definition} 
We show the existence of such sets for every $\gamma < 1$ in \cref{app:gamma-exploration-existence}.
\subsection{Generation without novelty (\texorpdfstring{$\gamma = 0$}{gamma = 0})}\label{subsec:genwithoutnovelty}
We first consider $\gamma = 0$, dropping the novelty constraint entirely, which serves as a useful baseline. In this regime, repeating the adversary's strings is allowed at any time, and
the generator can combine (i) parroting the adversary to keep precision high and (ii) sparse exploration to achieve perfect recall, even for partial exhaustions $\exhaustion{\advlang}$.
\begin{restatable}{theorem}{TheoremGenerationWithoutNovelty}\label{thm:nonoveltybound}
    Let $\coll$ be a countable collection of languages and $\targetlang\in\coll$ a target language.
    Fix an enumeration $\exhaustion{\targetlang}=\{\targetlang_n\}_{n=0}^\infty$ of $\targetlang$. There exists a generator $\gen$, such that, for any $\func$-bounded adversarial exhaustion $\exhaustion{\advlang} = \{\advlang_n\}_{n=0}^\infty$ with $\advlang_n \uparrow \advlang \subseteq \targetlang$, the generated exhaustion $\exhaustion{\guesslang}$ is $\func$-bounded and satisfies: 
    \begin{equation}
        \setlowerrecall{\exhaustion{\targetlang}}{\guesslang} = 1, \qquad \setlowerprecision{\targetlang}{\exhaustion{\guesslang}} = 1.
    \end{equation}
\end{restatable}

The proof is in \cref{app:gen-no-novelty-no-tail-precision}. Note that this theorem holds for both full and partial exhaustion of the target $\targetlang$.

\subsection{Generation with novelty
(\texorpdfstring{$\gamma = 1$}{gamma = 1})}
We now reinstate novelty (\Cref{def:novelty}), considering the other extreme of the spectrum, where it is not allowed to repeat any string that has been revealed before. Hence, the adversary can \emph{steal} strings by revealing them first, making perfect recall unattainable. Our goal is to understand how far recall can be pushed when we drop perfect tail precision, but aim to generate with precision $1$. Our generator combines two complementary mechanisms: 
\begin{enumerate}
\item \textbf{Exploration rounds} recover the non-stolen portion.
If a string is never revealed by the adversary (i.e., it lies in $\targetlang\setminus \advlang$),
then it is never blocked by novelty. Hence, an exploration routine that enumerates $\kleene{\alphabet}$ while skipping already-seen strings will eventually output it.

\item \textbf{Safe rounds} recover recall proportional to $\alpha$, when $\setlowerrecall{\exhaustion{\targetlang}}{\advlang}\geq \alpha$. We run a
\citet{kleinberg_language_2025}-style generator whose tail precision converges to $1$, which will be the key ingredient for achieving precision~$1$.
\end{enumerate}

In this setting, there can be up to two adversarial revelations
between two safe outputs. Therefore, we first need to adapt \citeposs{kleinberg_language_2025} algorithm to a so-called $k$-batched adversarial
setting (with $k=2$ here).

\begin{definition}
    Consider a setting in which the adversary is $k$-step bounded, i.e., it can reveal at most $k$ new strings per step $n$, and the generator is single-step bounded. We call this setting a $k$-batched adversarial setting.
\end{definition}
\Cref{alg:pods-generator} shows \citeposs{kleinberg_language_2025} algorithm for generation in the limit.\footnote{Note that they require $\exhaustion{\advlang}$ and $\exhaustion{\guesslang}$ to be single-step exhaustions.} At time $n$, consider all hypotheses that are consistent with the strings revealed by the adversary, i.e., all languages $\langs \in \coll$ with $\advlang_n \subseteq \langs$. Intersecting these consistent hypotheses in a global order, given by a fixed listing $(\langs'_i)_{i\ge1}$ of all languages in $\coll$, yields a decreasing \emph{intersection chain} $\chain_n$ (see \Cref{alg:update-chain}). Early sets in the chain are larger (better for recall), but riskier, while later sets are smaller (safer), but might miss many valid strings. The \emph{identified intersection} procedure (see \cref{alg:update-identified-intersection}) selects a particular member
$\intersect^{(n)}\in\chain_n$ that behaves stably over time: after some finite time, it is a subset of the true target $\targetlang$ (eventual validity), it is comparable across rounds, and it is \emph{full} (contains $\advlang$) infinitely often. Intuitively, $\intersect^{(n)}$ is designed to be a tractable candidate set of strings deemed safe given $\advlang_n$. However, outputting only from $\intersect^{(n)}$ can hurt recall, because $\intersect^{(n)}$ may shrink temporarily when the adversary reveals new strings, making the generator too conservative in those rounds. To mitigate this, the algorithm chooses an \emph{aggressive set} $\smash{\widetilde{\intersect}}^{(n)}$, which is either the current identified intersection or a carefully chosen larger fallback from recent rounds (see \Cref{alg:aggressive-set}). After some finite time—once the identified intersections are valid and comparable—this fallback simplifies to
\(
\smash{\widetilde{\intersect}}^{(n)} \in \{\intersect^{(n)},\,\intersect^{(n-1)}\}\subseteq \targetlang.
\)
Thus $\smash{\widetilde{\intersect}}^{(n)}$ is also eventually valid, but can be temporarily larger than $\intersect^{(n)}$, allowing the generator to improve recall. However, $\smash{\widetilde{\intersect}}^{(n)}$ itself evolves, so a string that appears safe at round $n$ might no longer be among preferred candidates at round $n+1$ if $\smash{\widetilde{\intersect}}^{(n)}$ shrinks. This problem is addressed by \defn{pods}. In each round $n$, the algorithm creates a pod $\Pod_n$ of size $s_n$ with the $s_n$ $\prec$-smallest \emph{unused} strings of $\smash{\widetilde{\intersect}}^{(n)}$ (excluding strings already revealed, output, or in previous pods) and adds those strings into a global \emph{pod pool}. In each round, the generator outputs the $\prec$-smallest element from the pool. This pool is crucial for achieving high recall: Even if $\smash{\widetilde{\intersect}}^{(n)}$ shrinks later, strings that were previously added
to pods are not lost and can eventually be output. 

We first generalize \citeposs{kleinberg_language_2025} pods algorithm (see \cref{alg:pods-generator}) and the corresponding recall bound for a $k$-batched setting, assuming that $\exhaustion{\advlang}$ is a $k$-step exhaustion.
\newpage

\begin{restatable}{lemma}{LemmaBatchedLower}
\label{lem:alpha-over-3}
Assume a $k$-batched adversarial setting, with $k\geq 2$. For any $\targetlang\in\coll$ and any $k$-batched exhaustion $\exhaustion{\advlang}$ of  $\advlang\subseteq \targetlang$ with lower recall $\setlowerrecall{\exhaustion{\targetlang}}{\advlang}$, the exhaustion $\exhaustion{\guesslang}$ generated by applying \Cref{alg:pods-generator} satisfies the novelty constraint and achieves
\begin{subequations}
\begin{equation}
\setlowerrecall{\exhaustion{\targetlang}}{\guesslang} \ge\ \frac{1}{k+1} \setlowerrecall{\exhaustion{\targetlang}}{\advlang},
\end{equation}
\begin{equation}
\setlowerprecision{\targetlang}{\exhaustion{\guesslang}} = 1, \qquad \setlowertailprecision{\targetlang}{\exhaustion{\guesslang}} = 1.
\end{equation}
\end{subequations}
\end{restatable}

\begin{proofsketch} First, note that \cref{alg:pods-generator} outputs only from the current pod pool $\podpool_n$, which is populated with $\prec$-small unused elements from aggressive sets. Thus, the generator is novel. Denote by $\podpool$ the set of all strings that are ever placed in a pod. We relate the adversary's revealed strings to the generator's outputs via a two-step charging argument.
First, consider strings in $\advlang$ that are neither output nor ever placed into a pod (not in $\podpool$).
Using that (after finite steps) the identified intersection is valid and is full infinitely often, one can show: for each such missed string, earlier pods contain many $\prec$-smaller unused candidates. Formally, we can charge each missed $\strx$ to a pod index $\rho(\strx)$ whose pod contains at least $s$ elements all $\prec \strx$, and because the adversary introduces at most $k$ new strings per round, each pod index receives at most $k$ charges.
Hence the total number of such missed strings up to $\targetlang_n$ is at most $\frac{k}{s}|\podpool\cap\targetlang_n|+\bigo(n)$. Second, we control the backlog in the pod pool. Since the algorithm outputs one element per round and the adversary reveals at most $k$ new strings per round, the number of adversary strings that can accumulate in the pool without being output is at most a factor $k$ larger than the number of outputs. We can charge every pooled string that appears in $\advlang$ but not in $\guesslang$ to a nearby generator output. Each output receives at most $k$ such charges. Combining this with the first charging step yields $|\advlang\cap\targetlang_n|\le (k+1+\bigo(1))|\guesslang\cap\targetlang_n|$. Dividing by $|\targetlang_n|$,
taking $\liminf$, and letting the pod-size parameter $s\to\infty$ gives
$\setlowerrecall{\exhaustion{\targetlang}}{\guesslang}\ge \frac{1}{k+1}\setlowerrecall{\exhaustion{\targetlang}}{\advlang}$. The tail precision and precision are $1$, as \citeposs{kleinberg_language_2025} algorithm guarantees that from a finite step onward every $\Delta\guesslang_n$ is in $\targetlang$.\looseness=-1
\end{proofsketch}
Full proof in \Cref{app:batched-adv}. The following lemma provides a matching upper bound on the lower recall of any novel generator, generalizing \citet[Theorem~3.3]{kleinberg_language_2025} to the $k$-batched setting. Together with \Cref{lem:alpha-over-3}, this shows that the factor $\frac{1}{k+1}$ is worst-case optimal.
\begin{restatable}{lemma}{LemmaBatchedAdversaryUpper}\label{lem:batched_adversary} Assume a $k$-batched adversarial setting. There exist countable collections of languages $\coll$, such that for a target $\targetlang$ and an adversarial exhaustion $\exhaustion{\advlang}$ with $\advlang_n \uparrow \advlang$, and $\advlang \subseteq \targetlang$, the recall of any novel generator that satisfies $\setlowertailprecision{\targetlang}{\exhaustion{\guesslang}} = 1$ is bounded by
    \begin{equation}
        \setlowerrecall{\exhaustion{\targetlang}}{\guesslang} \leq \frac{\setlowerrecall{\exhaustion{\targetlang}}{\advlang}}{k+1}.
    \end{equation}  
\end{restatable}

The full proof is given in \Cref{app:upperrecall-batched-adv}. We now move to examining the contribution of exploration rounds: if a generator eventually outputs every string in the target language that the adversary never reveals, then recall is at least $1-\beta$.

\begin{restatable}{lemma}{LemmaUpperDensity}\label{lem:upperdensitynoveltyalgo}
    Let $\targetlang \in \coll$ be the target, and $\exhaustion{\advlang}$ an exhaustion produced by the adversary with $\setupperrecall{\exhaustion{\targetlang}}{\advlang}$. Assume that $\exhaustion{\guesslang}$ contains every element in $\targetlang \setminus \advlang$ at some point. Then
    \begin{equation}
        \setlowerrecall{\exhaustion{\targetlang}}{\guesslang} \geq 1 - \setupperrecall{\exhaustion{\targetlang}}{\advlang}.
    \end{equation}
\end{restatable}
\begin{proof}
As $\targetlang \setminus \advlang \subseteq \guesslang$, we get
\begin{equation}
\frac{|\guesslang\cap\targetlang_n|}{|\targetlang_n|}
\ \ge\frac{|\targetlang \setminus \advlang\cap\targetlang_n|}{|\targetlang_n|} \geq
1-\frac{|\advlang\cap\targetlang_n|}{|\targetlang_n|}.
\end{equation}
Taking $\liminf_{n\to\infty}$ on the left and using $\liminf_{n\to\infty}(1-a_n)=1-\limsup_{n\to\infty}(a_n)$ gives
\begin{subequations}
\begin{align} 
\setlowerrecall{\exhaustion{\targetlang}}{\guesslang}
&\ge
1-\limsup_{n\to\infty}\frac{|\advlang\cap\targetlang_n|}{|\targetlang_n|}
\\&=
1-\setupperrecall{\exhaustion{\targetlang}}{\advlang}.
\end{align}
\end{subequations}
\end{proof}

We now present \cref{alg:gen-with-novelty}, which alternates safe rounds with exploration rounds at times $m\in\exploration$. On safe rounds, we execute one step of the batched \citet{kleinberg_language_2025} routine: update the consistent chain, add a new pod from the aggressive set to the global pod pool, and output the $\prec$-smallest unused pool element. On exploration rounds, we output the next string in the fixed enumeration $\exhaustion{E}$ of $\kleene{\alphabet}$, skipping previously seen strings to enforce novelty. Since exploration does not advance the safe routine, the safe routine perceives a $2$-batched adversary.

\begin{restatable}{theorem}{TheoremGenerationWithNovelty}
\label{thm:noveltybound}
For any $\targetlang\in\coll$ and any single-step adversarial exhaustion $\exhaustion{\advlang}$ with lower recall $\setlowerrecall{\exhaustion{\targetlang}}{\advlang}$ and upper recall $\setupperrecall{\exhaustion{\targetlang}}{\advlang}$, Algorithm~\ref{alg:gen-with-novelty} generating $\exhaustion{\guesslang} = \{\guesslang_n\}_{n=0}^\infty$ with $\guesslang_n \uparrow \guesslang$ achieves precision $1$ and recall
\begin{equation}
    \setlowerrecall{\exhaustion{\targetlang}}{\guesslang}\ge\ \max \Big( 1-\setupperrecall{\exhaustion{\targetlang}}{\advlang} , \frac{\setlowerrecall{\exhaustion{\targetlang}}{\advlang}}{3}\Big).
\end{equation}
\end{restatable}
\begin{proofsketch}
By construction, \Cref{alg:gen-with-novelty} eventually outputs every string in $\targetlang\setminus\advlang$,
which yields recall at least $1-\beta$ (cf.\ \Cref{lem:upperdensitynoveltyalgo}).
Moreover, the safe routine is invoked in a $2$-batched manner, implying an additional recall guarantee of at least $\alpha/3$ from the safe rounds (\Cref{lem:alpha-over-3}). These two contributions cannot, in general, be added: it may occur that both the safe and the exploration rounds output only strings from $\targetlang\setminus\advlang$.
Precision $1$ follows because the safe routine outputs only strings from $\targetlang$ after some finite time, and the exploration rounds are scheduled at a vanishing rate.
\end{proofsketch}

Full proof in \Cref{app:gen-with-novelty}. The term $1-\beta$ dominates when the adversary reveals only a small portion of the target along some subsequence (i.e., $\beta$ is small), since exploration can then recover almost all strings in $\targetlang$ that are never stolen. Overall, \Cref{alg:gen-with-novelty} can improve recall, particularly in regimes where $\alpha$ and $\beta$ are small. Finally, in the setting where the adversary presents a full exhaustion of $\targetlang$, the original \citet{kleinberg_language_2025} bound cannot be improved.

\subsection{Generation with
\texorpdfstring{$\gamma$}{gamma}-novelty for
\texorpdfstring{$\gamma\in[0,1)$}{gamma in [0,1)}}
With \Cref{alg:gen-with-gamma-novelty}, we provide an adaptation of \Cref{alg:gen-with-novelty} that satisfies $\gamma$-novelty for $\gamma \in [0,1)$ and leads to higher recall guarantees. Concretely, \Cref{alg:gen-with-gamma-novelty} differs from \Cref{alg:gen-with-novelty} in two ways: (i) the exploration set
$\exploration$ is required to be $\gamma$-admissible, and (ii) on exploration rounds, the algorithm skips only strings that have already been generated, rather than also skipping strings revealed by the adversary. Thus, an exploration output may be non-novel if it has previously been revealed by the adversary. Since $\exploration$ is $\gamma$-admissible, the number of such
non-novel outputs remains within the allowed quota. We show that for every $\gamma\in[0,1)$, precision and recall equal to $1$ can be achieved.

\begin{restatable}{theorem}{TheoremGenerationWithGammaNovelty}
\label{thm:gamma-noveltybound}
Fix any $\gamma\in[0,1)$. For any $\targetlang\in\coll$ and any single-step
adversarial exhaustion $\exhaustion{\advlang}$ with
$\advlang_n\uparrow\advlang\subseteq\targetlang$,
\Cref{alg:gen-with-gamma-novelty} generates a $\gamma$-novel exhaustion
$\exhaustion{\guesslang}=\{\guesslang_n\}_{n=0}^\infty$ with
$\guesslang_n\uparrow\guesslang$ such that
\begin{equation}
    \setlowerrecall{\exhaustion{\targetlang}}{\guesslang}=1,
    \qquad
    \setlowerprecision{\targetlang}{\exhaustion{\guesslang}}=1.
\end{equation}
\end{restatable}

\begin{proofsketch}
The modified exploration rounds enumerate all of $\kleene{\alphabet}$, skipping
only strings that have already been generated. Hence, every string in $\targetlang$ is eventually generated, which yields recall $1$. Precision is $1$ because safe rounds are eventually valid, while exploration rounds occur at
a vanishing rate. Finally, non-novel outputs can occur only in exploration rounds. Since the exploration set is $\gamma$-admissible, the fraction of novel outputs in $\guesslang_n$ for each $n$ is at least $\gamma$ and hence the generated exhaustion is $\gamma$-novel for every fixed $\gamma<1$. The full proof is given in \Cref{app:gen-with-gamma-novelty-without-tail-prec}.
\end{proofsketch}

\section{Conclusion}

We recast generation in the limit as a precision--recall problem. We introduce a precision metric based on asymptotic error rates, which remains informative even under infinitely many errors, and formalize eventual validity as the \emph{finite-time} stabilization of tail precision to $1$. We study relaxations of novelty and perfect tail precision, moving toward settings that better reflect large language model generation. To this end, we introduce an algorithmic template that interleaves safe rounds with sparse exploration rounds. We show that allowing hallucinations at a vanishing rate improves recall over eventually valid generators, but only when the adversary permanently withholds a large portion of the target language. In this regime, exploration rounds recover hidden strings while keeping the error rate controlled. Finally, we introduce a continuous novelty parameter and show that any fixed allowance of non-novel strings enables perfect precision and recall when perfect tail precision is relaxed.

\section*{Acknowledgements}
The authors thank Ron Jost and the anonymous reviewers for their helpful feedback and discussions. Alexandra Butoi is supported by the SNSF Grant No. 10.001.775. The authors declare no conflicts of interest.

\bibliography{custom_b}
\bibliographystyle{icml2026}

\newpage
\appendix

\section{Computability of exhaustions}
\label{app:comp-exhaustion}

We call an exhaustion $\exhaustion{\langs} \defeq \{\langs_n\}_{n=0}^\infty$ computable if there exists an algorithm (Turing machine) that emits $\langs_n$ in order. Note that an exhaustion may fail to be computable even if its limit language is.
Let $\langs$ be an infinite computable language, and write
$\str_1 \prec \str_2 \prec \cdots$ for the elements of $\langs$ in lexicographic order. Define an exhaustion $\exhaustion{\langs} \defeq \{\langs_n\}_{n=0}^\infty$ by $\langs_n \defeq \{\str_1,\ldots,\str_{\busybeaver{n}}\}$,
where $\busybeaver{n}$ denotes the busy beaver function.
If $\exhaustion{\langs}$ were computable, we could run a Turing machine emitting $\langs_n$ in order until it outputs the $n^{\text{th}}$ set. Returning its cardinality $|\langs_n|$ would compute the busy beaver function, contradicting the fact that it is not computable \citep{rado1962non}. However, if $\func$ and the limit language $\langs$ are computable, one can construct a computable $\func$-bounded exhaustion of $\langs$, provided that $\func$ allows for enough cumulative capacity.  

We denote the set of computable exhaustions of $\langs$ by $\exhaustioncompset{\langs}$. Computable exhaustions imply that the language is recursively enumerable, as shown in the following lemma.
\begin{lemma}
Let $\langs \subseteq \kleene{\alphabet}$ and let $\exhaustion{\langs}=\{\langs_n\}_{n=0}^\infty$ be an exhaustion of $\langs$.
If $\exhaustion{\langs}$ is computable, then $\langs$ is recursively enumerable.
\end{lemma}

\begin{proof} We will construct a Turing Machine that enumerates $\langs$. Since $\exhaustion{\langs}$ is computable, there exists a Turing machine $M$ that emits (in order) a finite description of
$\langs_1,\langs_2,\ldots$. This can be each $\langs_n$ output as a finite list of strings with separators and an end marker.

We construct an enumerator $E$ for $\langs$ as follows.
Run $M$. Whenever $M$ finishes emitting the description of $\langs_n$, decode this finite description and print every string in $\langs_n$
(to an output tape), then continue running $M$ to obtain $\langs_{n+1}$, and so on.

We show that $E$ enumerates exactly $\langs$.

(\emph{Soundness}) Every string printed by $E$ belongs to some $\langs_n$, hence belongs to
$\bigcup_{k=1}^\infty \langs_k = \langs$.

(\emph{Completeness}) Let $\strx \in \langs$. Since $\langs=\bigcup_{n=1}^\infty \langs_n$, there exists $N$ such that $\strx \in \langs_N$.
When $M$ emits $\langs_N$, the machine $E$ prints all elements of $\langs_N$, in particular $\strx$.
Thus every element of $\langs$ is eventually printed by $E$.

Therefore $E$ lists exactly the elements of $\langs$ (possibly with repeats), so $\langs$ is recursively enumerable.
\end{proof}

\newpage
\section{Precision}\label{app:precision}
\subsection{Exhaustion-level precision} \label{app:exhaustion-precision}
\begin{example}[Precision need not exist]\label{ex:precision-no-limit}

Let $\alphabet=\{\syma,\symb\}$. Define the target language
$\targetlang \defeq \{\syma^n \mid n\in\N\}$ with an exhaustion $\exhaustion{\targetlang}$,
\begin{equation}
\targetlang_n \defeq \{\syma^i \mid 1\le i\le n\}.
\end{equation}
Let the guess language be $\guesslang \defeq \targetlang \cup \{\symb^n \mid n\in\N\}$ and define an exhaustion
\begin{equation}
\guesslang_n \defeq \{\syma^1,\dots,\syma^n\}\ \cup\ \{\symb^1,\dots,\symb^{2^{\lfloor \log_2 n\rfloor}}\},
\end{equation}
i.e., at time $n$ we include all $\syma$-strings up to length $n$ and all $\symb$-strings up to the largest power of $2$ not exceeding $n$.
Then $\guesslang_n \subseteq \guesslang_{n+1}$ and $\guesslang_n\uparrow\guesslang$.
Because $\targetlang_n \subseteq \guesslang_n$ and the $\syma$-strings and $\symb$-strings are disjoint,

\noindent
\begin{minipage}{0.52\linewidth}
\begin{equation}
|\targetlang_n\cap\guesslang_n| = |\targetlang_n| = n,
\end{equation}
\end{minipage}\hfill
\begin{minipage}{0.46\linewidth}
\begin{equation}
|\guesslang_n| = n + 2^{\lfloor \log_2 n\rfloor},
\end{equation}
\end{minipage}

and therefore, we obtain

\begin{equation}\label{eq:pn-def}
\precisionratio \defeq \frac{|\targetlang_n\cap\guesslang_n|}{|\guesslang_n|},
\end{equation}
\begin{equation}\label{eq:pn-closed}
\precisionratio = \frac{n}{n+2^{\lfloor \log_2 n\rfloor}}.
\end{equation}
Evaluating $\precisionratio$ along two subsequences gives different limit points, as we can see below.
\begin{subequations}
\begin{align}
\lim_{m \to \infty}p_{2^m}&=\lim_{m \to \infty}\frac{2^m}{2^m+2^m}=\frac12,
\\
\lim_{m \to \infty} p_{2^{m+1}-1}
&=\lim_{m \to \infty} \frac{2^{m+1}-1}{(2^{m+1}-1)+2^m}=\frac23.
\end{align}
\end{subequations}
Thus, by the definitions of $\liminf$ and $\limsup$,
\begin{subequations}
\begin{align}
\lowerprecision{\exhaustion{\targetlang}}{\exhaustion{\guesslang}}
&=\liminf_{n\to\infty} \precisionratio \leq \frac12,
\\\upperprecision{\exhaustion{\targetlang}}{\exhaustion{\guesslang}}
&=\limsup_{n\to\infty} \precisionratio \geq \frac23,
\end{align}
\end{subequations}
so the limit defining precision does not exist.
\end{example}

\subsection{Membership-based precision}\label{app:membership-precision}
\begin{restatable}{lemma}{LemmaLowerPrecision}\label{lem:lower-precision}
Let $\targetlang\subseteq \kleene{\alphabet}$ be the target language and let $\exhaustion{\guesslang} = \{\guesslang_n\}_{n=0}^\infty$ be an exhaustion of the guess language $\guesslang$. Then, 
\begin{equation}\label{eq:best-lower-precision}
\sup_{\exhaustion{\targetlang}\in\exhaustionset{\targetlang}} \lowerprecision{\exhaustion{\targetlang}}{\exhaustion{\guesslang}} = \liminf_{n\to\infty}\frac{|\targetlang\cap \guesslang_n|}{|\guesslang_n|}.
\end{equation}
\end{restatable}
\begin{proof}
We prove the lemma in two parts.

\textbf{Part 1 ($\sup_{\exhaustion{\targetlang}\in\exhaustionset{\targetlang}}  \lowerprecision{\exhaustion{\targetlang}}{\exhaustion{\guesslang}} \leq \liminf_{n\to\infty}\frac{|\targetlang\cap \guesslang_n|}{|\guesslang_n|}$).}
Choose $\exhaustion{\targetlang}\in\exhaustionset{\targetlang}$ arbitrarily. 
Because $\targetlang_n\subseteq \targetlang$ for all $n$, we have $|\targetlang_n\cap \guesslang_n|\le |\targetlang\cap \guesslang_n|$ for
every $n$. 
Dividing by $|\guesslang_n|$ and taking $\liminf$ yields
\begin{equation}
\lowerprecision{\exhaustion{\targetlang}}{\exhaustion{\guesslang}}
=
\liminf_{n\to\infty}\frac{|\targetlang_n\cap \guesslang_n|}{|\guesslang_n|}
\le
\liminf_{n\to\infty}\frac{|\targetlang\cap \guesslang_n|}{|\guesslang_n|}.
\end{equation}
Taking the supremum over $\exhaustion{\targetlang}\in\exhaustionset{\targetlang}$ proves
\begin{equation}
\sup_{\exhaustion{\targetlang}\in\exhaustionset{\targetlang}} \ \lowerprecision{\exhaustion{\targetlang}}{\exhaustion{\guesslang}}
\le
\liminf_{n\to\infty}\frac{|\targetlang\cap \guesslang_n|}{|\guesslang_n|}.
\end{equation}

\textbf{Part 2 ($\liminf_{n\to\infty}\frac{|\targetlang\cap \guesslang_n|}{|\guesslang_n|} \leq \sup_{\exhaustion{\targetlang}\in \exhaustionset{\targetlang}}\lowerprecision{\exhaustion{\targetlang}}{\exhaustion{\guesslang}}$).}

Fix an enumeration $\exhaustion{Y} = \{Y_n\}_{n=0}^\infty $ with $Y_n \uparrow \targetlang$.
We define an exhaustion $\exhaustion{\targetlang}^\star$. 
For each $n \in \N$, define 
\begin{equation}
\targetlang_n^\star \defeq (\targetlang\cap \guesslang_n)\ \cup\ Y_n.
\end{equation}
Note that each $\targetlang_n^\star$ is finite, $\targetlang_n^\star\subseteq \targetlang_{n+1}^\star$, and
$\targetlang_n^\star\uparrow\targetlang$.
Thus, $\exhaustion{\targetlang}^\star\in\exhaustionset{\targetlang}$.
Moreover,
\begin{equation}\label{eq:star-intersection}
\targetlang_n^\star\cap \guesslang_n
=
\bigl((\targetlang\cap \guesslang_n)\cup Y_n\bigr)\cap \guesslang_n
=
(\targetlang\cap \guesslang_n)\cup(Y_n\cap \guesslang_n)
=
(\targetlang  \cup Y_n) \cap \guesslang_n
=
\targetlang\cap \guesslang_n.
\end{equation}
where the last equality uses $Y_n \subseteq \targetlang$ for all $n \geq 0$.
Hence, $\frac{|\targetlang_n^\star\cap \guesslang_n|}{|\guesslang_n|}
=
\frac{|\targetlang\cap \guesslang_n|}{|\guesslang_n|}$ for all $n \geq 0$.
Taking $\liminf$ yields
\begin{equation}
\lowerprecision{\exhaustion{\targetlang}^\star}{\exhaustion{\guesslang}}
=
\liminf_{n\to\infty}\frac{|\targetlang_n^\star\cap \guesslang_n|}{|\guesslang_n|}
=
\liminf_{n\to\infty}\frac{|\targetlang\cap \guesslang_n|}{|\guesslang_n|},
\end{equation}
which proves the reverse inequality in~\Cref{eq:best-lower-precision} as $\lowerprecision{\exhaustion{\targetlang}^\star}{\exhaustion{\guesslang}} \leq \sup_{\exhaustion{\targetlang}\in\exhaustionset{\targetlang}} \lowerprecision{\exhaustion{\targetlang}}{\exhaustion{\guesslang}} $.
\end{proof}

An analogous property holds for upper precision.

\begin{restatable}{lemma}{LemmaUpperPrecision}\label{lem:upper-precision}
Let $\targetlang\subseteq \kleene{\alphabet}$ be the target language and let $\exhaustion{\guesslang} = \{\guesslang_n\}_{n=0}^\infty$ be an exhaustion of the guess language $\guesslang$. Then, 
\begin{equation}\label{eq:best-upper-precision}
\sup_{\exhaustion{\targetlang}\in\exhaustionset{\targetlang}} \upperprecision{\exhaustion{\targetlang}}{\exhaustion{\guesslang}} = \limsup_{n\to\infty}\frac{|\targetlang\cap \guesslang_n|}{|\guesslang_n|}.
\end{equation}
\end{restatable}
\begin{proof}
    Analogous to proof of \Cref{lem:lower-precision}.
\end{proof}

\TheoremLowerPrecisionExhaustion*
\begin{proof}
We prove the theorem in two parts.

\textbf{Part 1 ($\sup_{\exhaustion{\targetlang}\in \exhaustionsetbounded{\targetlang}}  \lowerprecision{\exhaustion{\targetlang}}{\exhaustion{\guesslang}} \leq \liminf_{n\to\infty}\frac{|\targetlang\cap \guesslang_n|}{|\guesslang_n|}$).} 
This is analogous to the proof of Part 1 in \cref{lem:lower-precision}.  

\textbf{Part 2 ($\liminf_{n\to\infty}\frac{|\targetlang\cap \guesslang_n|}{|\guesslang_n|}\le \sup_{\exhaustion{\targetlang}\in \exhaustionsetbounded{\targetlang}} \lowerprecision{\exhaustion{\targetlang}}{\exhaustion{\guesslang}}$).}
Fix an enumeration $\exhaustion{Z} = \{Z_n\}_{n=0}^\infty$ of $\targetlang\setminus \guesslang$.\footnote{If $\targetlang \setminus \guesslang$ is finite, we can fix a single-step exhaustion that converges to a finite language.}  
Let $h_n \defeq |\Delta \guesslang_n \setminus \targetlang|$ and $H_n = \sum_{i=1}^n h_i$.
We construct $\exhaustion{\targetlang^\star}=\{\targetlang_n^\star\}_{n=0}^\infty\in \exhaustionsetbounded{\targetlang}$ and distinguish two cases.

\medskip
\fbox{%
  \begin{minipage}{0.28\linewidth}
  \textbf{Case 1.} $\lim_{n\to\infty} H_n = \infty$.
  \end{minipage}%
}

Define for each $n\ge 0$,
\begin{equation}\label{eq:Tstar-case1-slack}
\targetlang_n^\star \defeq (\targetlang\cap \guesslang_n)\ \cup Z_{
H_n}.
\end{equation}

\fbox{%
  \begin{minipage}{0.48\linewidth}
  \textbf{Claim 1.} $\targetlang_n^\star$ is an $\func$-bounded exhaustion of $\targetlang$.
  \end{minipage}%
}

To start, note $\targetlang_n^\star$ is finite and $\targetlang_n^\star\subseteq \targetlang$.
Moreover, $\targetlang_n^\star\subseteq \targetlang_{n+1}^\star$ since $\guesslang_n\subseteq \guesslang_{n+1}$ implies $\targetlang\cap \guesslang_n\subseteq \targetlang\cap \guesslang_{n+1}$,
and since $H_n\le H_{n+1}$ holds. 
Furthermore, by construction
\begin{equation}
|\Delta\targetlang_n^\star|
=
|\targetlang\cap \Delta\guesslang_n| + h_n \leq f(n).
\end{equation}
Finally, because $\limsup H_n=\infty$, every $Z_k$ appears in $\targetlang_n^\star$ for sufficiently large $n$.
Also, $(\targetlang\cap \guesslang_n)\uparrow (\targetlang\cap \guesslang)$.
Hence $\targetlang_n^\star\uparrow \targetlang$ as desired.

\fbox{%
  \begin{minipage}{0.32\linewidth}
  \textbf{Claim 2.} $\targetlang_n^\star\cap \guesslang_n = \targetlang\cap \guesslang_n$.
  \end{minipage}%
}

Finally, for every $n$,
\begin{equation}\label{eq:case1-intersection-slack}
\targetlang_n^\star\cap \guesslang_n = (\targetlang\cap \guesslang_n) \cup (Z_{H_n}\cap \guesslang_n) = \targetlang\cap \guesslang_n,
\end{equation}
because $Z_{H_n}\subseteq \targetlang\setminus \guesslang$ is disjoint from $\guesslang_n$. We have that
\begin{equation}
\lowerprecision{\exhaustion{\targetlang}^\star}{\exhaustion{\guesslang}}
=
\liminf_{n\to\infty}\frac{|\targetlang_n^\star\cap \guesslang_n|}{|\guesslang_n|}
=
\liminf_{n\to\infty}\frac{|\targetlang\cap \guesslang_n|}{|\guesslang_n|}.
\end{equation}

\medskip
\fbox{%
  \begin{minipage}{0.3\linewidth}
  \textbf{Case 2.} $\limsup_n H_n < \infty$.
  \end{minipage}%
}

First, we argue that the following $\liminf$ is 1.
\begin{equation}\label{eq:case2-rhs}
\liminf_{n\to\infty}\frac{|\targetlang\cap \guesslang_n|}{|\guesslang_n|}
=\liminf_{n\to\infty}\left(1-\frac{H_n}{|\guesslang_n|}\right)
=1.
\end{equation}
This holds as $\limsup_n H_n < \infty$ and $\exhaustion{\guesslang}$ exhausts an infinite language $\guesslang$.
Next, we show that we can construct an exhaustion that achieves the limit.
Choose an exploration set $\exploration \subset \N$ according to \Cref{def:exploration-set} and write
$\explorationidx(m)\defeq |\exploration\cap\{1,\dots,m\}|$.

Fix an enumeration $\exhaustion{Z}=\{Z_j\}_{j=0}^\infty$ of $\targetlang\setminus \guesslang$ with $Z_0\defeq\emptyset$.
Fix an enumeration $\exhaustion{Y}=\{Y_j\}_{j=0}^\infty$ of $\targetlang\cap \guesslang$ such that
\begin{equation}\label{eq:Y-refines-A}
Y_{|\guesslang_n\cap \targetlang|}=\guesslang_n\cap \targetlang\qquad\text{for all $n$.}
\end{equation}
Note that the enumerations are disjoint.

Let $m_n\defeq |\guesslang_n|$. We define the exhaustion $\exhaustion{\targetlang}^\star=\{\targetlang_n^\star\}_{n=0}^\infty$ by
\begin{equation}\label{eq:Tstar-def}
\targetlang_n^\star \defeq Y_{m_n-\explorationidx(m_n)}\ \cup\ Z_{\explorationidx(m_n)}.
\end{equation}
Then $|\targetlang_n^\star|=(m_n-\explorationidx(m_n))+\explorationidx(m_n)=m_n=|\guesslang_n|$, and thus $|\Delta\targetlang_n^\star| = m_n - m_{n-1} \leq \func(n)$. Furthermore, $\targetlang_n^\star\subseteq \targetlang$ for all $n$.
Moreover, since $m_n$, $\explorationidx(m_n)$ and $m_n - \explorationidx(m_n)$ are nondecreasing, also $\targetlang_n^\star\subseteq \targetlang_{n+1}^\star$.
Finally,  $m_n\to\infty$, $m_n -\explorationidx(m_n) \to \infty$ and $\explorationidx(m_n)\to\infty$ (because $\exploration$ is infinite, but contains no consecutive integers), we have
$\targetlang_n^\star\uparrow \targetlang$.
As $H_n$ is bounded and $\explorationidx(m_n) \to \infty$, for large enough $n$ we have
\begin{equation}\label{eq:case2-intersection}
|\targetlang_n^\star\cap \guesslang_n|
\ \ge\
|Y_{\,m_n-\explorationidx(m_n)}|
\ =\
m_n-\explorationidx(m_n).
\end{equation}
Dividing by $m_n=|\guesslang_n|$ and taking $\liminf$ yields
\begin{align}
\lowerprecision{\exhaustion{\targetlang^\star}}{\exhaustion{\guesslang}}
&=
\liminf_{n\to\infty}\frac{|\targetlang_n^\star\cap \guesslang_n|}{|\guesslang_n|}
\\ &\ge\
\liminf_{n\to\infty}\left(1-\frac{\explorationidx(m_n)}{m_n}\right)
=1,
\end{align}
using $\explorationidx(m_n)/m_n\to 0$.
Together with \Cref{eq:case2-rhs}, this yields
$\lowerprecision{\exhaustion{\targetlang}^\star}{\exhaustion{\guesslang}}=1$.
\end{proof}

\newpage
\section{Recall}
\subsection{Exhaustion-level recall}
\label{app:exhaustion-recall}

\begin{example}[Recall need not exist]\label{ex:recall-no-limit}
Let $\alphabet=\{\syma\}$. Define the target language $\targetlang \defeq \{\syma^n \mid n\in\N\}$ with exhaustion $\exhaustion{\targetlang}$ given by
\begin{equation}
\targetlang_n \defeq \{\syma^i \mid 1\le i\le n\}.
\end{equation}
Define the guess language $\guesslang\subseteq\targetlang$ by alternating blocks: for each $m\ge 1$ let
\begin{equation}
I_m \defeq \bigl(2^{2m-1},\,2^{2m}\bigr]
\end{equation}
and
\begin{equation}
E_m \defeq \bigl(2^{2m},\,2^{2m+1}\bigr].
\end{equation}
Then include all $\syma^n$ with $n \in I_m$ and exclude $\syma^n$ with $n \in E_m$. Formally, 
\begin{equation}
\guesslang \defeq \{\syma^n \mid n\in \bigcup_{m\ge1} I_m\}.
\end{equation}
Let $\exhaustion{\guesslang}$ be the canonical exhaustion induced by $\exhaustion{\targetlang}$,
\begin{equation}
\guesslang_n \defeq \guesslang \cap \targetlang_n.
\end{equation}
Then $\guesslang_n\subseteq \guesslang_{n+1}$ and  $\guesslang_n \uparrow \guesslang$.

Since $\guesslang_n\subseteq\targetlang_n$, we have
\begin{equation}
|\targetlang_n\cap\guesslang_n| = |\guesslang_n|,
\qquad
|\targetlang_n| = n.
\end{equation}
Defining
\begin{equation}
r_n \defeq \frac{|\targetlang_n\cap\guesslang_n|}{|\targetlang_n|},
\end{equation}
we obtain
\begin{equation}
r_n = \frac{|\guesslang_n|}{n}.
\end{equation}

Let $S_m \defeq \sum_{j=1}^m |I_j| = \sum_{j=1}^m 2^{2j-1} = \frac{2}{3}\bigl(4^m-1\bigr)$.
Then at the end of an included block we have $|\guesslang_{2^{2m}}|=S_m$, hence
\begin{equation}
r_{2^{2m}}
=\frac{S_m}{2^{2m}}
=\frac{\frac{2}{3}(4^m-1)}{4^m}
=\frac{2}{3}\Bigl(1-4^{-m}\Bigr)
\xrightarrow[m\to\infty]{}\frac{2}{3}.
\end{equation}
At the end of the subsequent excluded block the numerator is still $S_m$, so
\begin{equation}
r_{2^{2m+1}}
=\frac{S_m}{2^{2m+1}}
=\frac{\frac{2}{3}(4^m-1)}{2\cdot 4^m}
=\frac{1}{3}\Bigl(1-4^{-m}\Bigr)
\xrightarrow[m\to\infty]{}\frac{1}{3}.
\end{equation}
Thus $r_n$ has two different limit points, and therefore, by the definitions of $\liminf$ and $\limsup$,
\begin{align}
\lowerrecall{\exhaustion{\targetlang}}{\exhaustion{\guesslang}}
&=\liminf_{n\to\infty} r_n \le \frac{1}{3},
\\
\upperrecall{\exhaustion{\targetlang}}{\exhaustion{\guesslang}}
&=\limsup_{n\to\infty} r_n \ge \frac{2}{3},
\end{align}
so the limit defining recall does not exist.
\end{example}

\subsection{Coverage-based recall}
\label{app:coverage-recall}
\LemmaLowerRecall*
\begin{proof}
We prove the lemma in two parts.

\textbf{Part 1 ($\sup_{\exhaustion{\guesslang}\in\exhaustionset{\guesslang}} \lowerrecall{\exhaustion{\targetlang}}{\exhaustion{\guesslang}} \leq \liminf_{n\to\infty}\frac{|\targetlang_n\cap \guesslang|}{|\targetlang_n|}$).}
Choose $\exhaustion{\guesslang}\in\exhaustionset{\guesslang}$ arbitrarily. 
Because $\guesslang_n\subseteq \guesslang$ for all $n$, we have $|\targetlang_n\cap \guesslang_n|\le |\targetlang_n\cap \guesslang|$ for
every $n$. 
Dividing by $|\targetlang_n|$ and taking $\liminf$ yields
\begin{equation}
\lowerrecall{\exhaustion{\targetlang}}{\exhaustion{\guesslang}}
=
\liminf_{n\to\infty}\frac{|\targetlang_n\cap \guesslang_n|}{|\targetlang_n|}
\le
\liminf_{n\to\infty}\frac{|\targetlang_n\cap \guesslang|}{|\targetlang_n|}.
\end{equation}
Taking the supremum over $\exhaustion{\guesslang}\in\exhaustionset{\guesslang}$ proves
\begin{equation}
\sup_{\exhaustion{\guesslang}\in\exhaustionset{\guesslang}} \ \lowerrecall{\exhaustion{\targetlang}}{\exhaustion{\guesslang}}
\le
\liminf_{n\to\infty}\frac{|\targetlang_n\cap \guesslang|}{|\targetlang_n|}.
\end{equation}

\textbf{Part 2 ($\liminf_{n\to\infty}\frac{|\targetlang_n\cap \guesslang|}{|\targetlang_n|} \leq \sup_{\exhaustion{\guesslang}\in\exhaustionset{\guesslang}}\lowerrecall{\exhaustion{\targetlang}}{\exhaustion{\guesslang}}$).}

To show  the reverse inequality, fix an exhaustion $\exhaustion{Y} = \{Y_n\}_{n=0}^\infty$ with $\exhaustion{Y}\uparrow\guesslang$. 
Then, we define the following exhaustion $\exhaustion{\guesslang^\star}$ of $\guesslang$: for each $n \in \N$, let 
\begin{equation}
\guesslang_n^\star \defeq (\guesslang\cap \targetlang_n)\ \cup\ Y_n.
\end{equation}
Then each $\guesslang_n^\star$ is finite, $\guesslang_n^\star\subseteq \guesslang_{n+1}^\star$, and
$\guesslang_n^\star \uparrow \guesslang$, so $\exhaustion{\guesslang}^\star\in\exhaustionset{\guesslang}$.
Moreover,
\begin{align}
\guesslang_n^\star\cap \targetlang_n
&=
\bigl((\guesslang\cap \targetlang_n)\cup Y_n\bigr)\cap \targetlang_n
\\&=
(\guesslang\cap \targetlang_n)\cup(Y_n\cap \targetlang_n)
\\&=
\guesslang\cap \targetlang_n,
\end{align}
where the last equality uses $Y_n\cap \targetlang_n\subseteq \guesslang\cap \targetlang_n$.
Hence, for all $n$,
\begin{equation}
\frac{|\guesslang_n^\star\cap \targetlang_n|}{|\targetlang_n|}
=
\frac{|\guesslang\cap \targetlang_n|}{|\targetlang_n|}.
\end{equation}
Taking $\liminf$ yields
\begin{equation}
\lowerrecall{\exhaustion{\targetlang}}{\exhaustion{\guesslang}^\star}
=
\liminf_{n\to\infty}\frac{|\guesslang_n^\star\cap \targetlang_n|}{|\targetlang_n|}
=
\liminf_{n\to\infty}\frac{|\guesslang\cap \targetlang_n|}{|\targetlang_n|},
\end{equation}
which proves the reverse inequality in~\Cref{eq:best-lower-recall}.
\end{proof}

The same property can be proven for upper recall in an analogous way.

\begin{restatable}{lemma}{LemmaUpperRecall}\label{lem:upper-recall}
Let $\targetlang\subseteq \kleene{\alphabet}$ be the target language and let $\exhaustion{\guesslang} = \{\guesslang_n\}_{n=0}^\infty$ be an exhaustion of the guess language $\guesslang$. Then, 
\begin{equation}\label{eq:best-upper-recall}
\sup_{\exhaustion{\guesslang}\in\exhaustionset{\guesslang}} \upperrecall{\exhaustion{\targetlang}}{\exhaustion{\guesslang}} = \limsup_{n\to\infty}\frac{|\targetlang_n\cap \guesslang|}{|\targetlang_n|}.
\end{equation}
\end{restatable}

Furthermore, we formulate a theorem analogous to \Cref{th:lower-precision-bounded-exhaustions} for the recall.

\begin{restatable}{theorem}{TheoremLowerRecallExhaustion}\label{th:lower-recall-bounded-exhaustions} Let $\func(n) \colon \N \to \N$ be a function with $\lim_{N\to \infty}\sum_{n=1}^N \func(n) = \infty$. Let $\guesslang\subseteq \kleene{\alphabet}$ be the guess language and let $\exhaustion{\targetlang} = \{\targetlang_n\}_{n=0}^\infty \in \exhaustionsetbounded{\targetlang}$ be an $\func$-bounded exhaustion of the target language $\targetlang$. Then, 
\begin{equation}\label{eq:best-lower-recall-exhaustion}
\sup_{\exhaustion{\guesslang}\in\exhaustionsetbounded{\guesslang}} \lowerrecall{\exhaustion{\targetlang}}{\exhaustion{\guesslang}} = \liminf_{n\to\infty}\frac{|\targetlang_n\cap \guesslang|}{|\targetlang_n|}
\end{equation}
\end{restatable}
\begin{proof}
    Analogous to proof of \Cref{th:lower-precision-bounded-exhaustions}.
\end{proof}
\newpage
\section{Tail precision} \label{app:tail-precision}
\PropositionConvergenceTailPrecision*

\begin{proof}
For each $n$, recall the definition of the step-wise tail precision:
\begin{equation}
t_n \defeq 
\begin{cases}
1, & \text{if } |\Delta \guesslang_n|=0,\\[2pt]
\frac{|\targetlang\cap \Delta\guesslang_n|}{|\Delta\guesslang_n|}, & \text{otherwise.}
\end{cases}
\end{equation}
Assume first that at each step at most one new string is added, i.e.\ $|\Delta \guesslang_n|\le 1$ (single-step exhaustion).
Then $t_n\in\{0,1\}$ for all $n$.
Let $E \defeq \{n\in\N : t_n = 0\}$ be the set of error steps.
If $E$ is infinite, then $t_n=0$ occurs infinitely often and thus
\begin{equation}
\liminf_{n\to\infty} t_n = 0.
\end{equation}
If $E$ is finite, let $N_0 > \max E$. Then $t_n=1$ for all $n\ge N_0$, and hence
\begin{equation}
\liminf_{n\to\infty} t_n = 1.
\end{equation}
In either case, $\setlowertailprecision{\targetlang}{\exhaustion{\guesslang}}
= \liminf_{n\to\infty} t_n$ is reached after finitely many steps and equals either $0$ or $1$.

More generally, suppose there is a constant $c\ge 1$ such that $|\Delta \guesslang_n|\le c$ for all $n$.
Then for every $n$ we have
\begin{equation}
t_n \in F_c \defeq \Bigl\{\frac{i}{j} \colon i,j \in \N_0, 1\le j \le c, 0 \le i\le j\Bigr\}.
\end{equation}
Let $\ell \defeq \liminf_{n\to\infty} t_n$. Since $t_n$ takes values in the finite set $F_c$,
there exists a minimum value among those that occur infinitely often; moreover this minimum equals $\ell$.
In particular, $\ell\in F_c$ and $t_n=\ell$ occurs infinitely often.

Let $B \defeq \{x\in F_c : x<\ell\}$. Every value $x\in B$ can occur only finitely many times,
otherwise we would have $\liminf_{n\to\infty} t_n \le x < \ell$, a contradiction.
Hence there exists $N_0$ such that $t_n\notin B$ for all $n\ge N_0$, i.e.\ $t_n\ge \ell$ for all $n\ge N_0$.
Since $t_n=\ell$ occurs infinitely often, there is also some $n\ge N_0$ with $t_n=\ell$.
Therefore,
\begin{equation}
\inf_{m\ge N_0} t_m = \ell,
\end{equation}
so the lower tail precision $\setlowertailprecision{\targetlang}{\exhaustion{\guesslang}}
= \ell$ is attained from time $n \geq N_0$ onward.
\end{proof}

\PropositionValidityTailPrecision*

\begin{proof}
Recall the definition of the step-wise tail precision:
\begin{equation}
t_n \defeq
\begin{cases}
1, & |\Delta \guesslang_n|=0,\\[2mm]
\frac{|\targetlang \cap \Delta \guesslang_n|}{|\Delta \guesslang_n|}, & |\Delta \guesslang_n|>0.
\end{cases}
\end{equation}
By definition, $\setlowertailprecision{\targetlang}{\exhaustion{\guesslang}}=\liminf_{n\to\infty} t_n$.

\smallskip\noindent
($\Rightarrow$)
If $\gen$ is eventually valid, then there exists $n^\star$ such that for all $n\ge n^\star$,
\begin{equation}
\gen(\advlang_n,\guesslang_{n-1})\subseteq \targetlang.
\end{equation}
Equivalently, for all $n\ge n^\star$, every newly generated string lies in $\targetlang$, i.e.,
$\Delta \guesslang_n\subseteq \targetlang$.
Thus $t_n=1$ for all $n\ge n^\star$, which means the tail-precision sequence attains $1$ from some finite time onward.

\smallskip\noindent
($\Leftarrow$)
Conversely, suppose $\setlowertailprecision{\targetlang}{\exhaustion{\guesslang}} = 1$ is attained in finite time. This implies that there exists $n^\star$ such that $t_n=1$ for all $n\ge n^\star$.
For any such $n$, if $|\Delta \guesslang_n|>0$ then
\begin{equation}
\frac{|\Delta \guesslang_n\cap \targetlang|}{|\Delta \guesslang_n|}=1
\quad\Longrightarrow\quad
\Delta \guesslang_n\subseteq \targetlang,
\end{equation}
and if $|\Delta \guesslang_n|=0$ then there is nothing to show.
Hence $\Delta \guesslang_n\subseteq \targetlang$ for all $n\ge n^\star$, i.e.,
$\gen(\advlang_n,\guesslang_{n-1})\subseteq \targetlang$ for all $n\ge n^\star$.
This is exactly eventual validity.
\end{proof}

\begin{lemma}[Tail precision $1$ implies precision $1$]
\label{lem:tail-precision-implies-precision}
Let $\targetlang$ be the target language and let
$\exhaustion{\guesslang}=\{\guesslang_n\}_{n=0}^\infty$ be an exhaustion of $\guesslang$. Then
\begin{equation}
    \setlowertailprecision{\targetlang}{\exhaustion{\guesslang}} = 1 \implies \setlowerprecision{\targetlang}{\exhaustion{\guesslang}} = 1.
\end{equation}
\end{lemma}
\begin{proof}
For $n\ge 1$, define
\begin{equation}
s_n \defeq |\Delta \guesslang_n|,
\qquad
e_n \defeq |\Delta \guesslang_n \setminus \targetlang|.
\end{equation}

If $s_n = 0$, let $t_n = 1$, otherwise let
\begin{equation}
t_n \defeq \frac{|\Delta \guesslang_n \cap \targetlang|}{|\Delta \guesslang_n|}
= 1 - \frac{e_n}{s_n}
\qquad\text{and}\qquad
p_n \defeq \frac{|\guesslang_n \cap \targetlang|}{|\guesslang_n|}
= 1 - \frac{|\guesslang_n \setminus \targetlang|}{|\guesslang_n|}.    
\end{equation}

The assumption $\setlowertailprecision{\targetlang}{\exhaustion{\guesslang}}=1$ means
$\liminf_{n\to\infty} t_n = 1$. Since always $t_n\le 1$, this implies $\lim_{n\to \infty} t_n = 1$, hence
\begin{equation}
    \lim_{n\to \infty}\frac{e_n}{s_n} =0.
\end{equation}

Now let $E_n \defeq |\guesslang_n \setminus \targetlang|$ be the cumulative number of errors up to time $n$. We have
\begin{equation}
    E_n = \sum_{i=1}^n e_i
\qquad\text{and}\qquad
|\guesslang_n| = \sum_{i=1}^n s_i.
\end{equation}
Fix $\varepsilon>0$. Since $\lim_{n\to \infty}e_n/s_n= 0$, there exists $N$ such that $e_n \le \varepsilon s_n$ for all $n\ge N$.
Then for any $n\ge N$,
\begin{equation}
E_n
= \sum_{i=1}^{N-1} e_i + \sum_{i=N}^n e_i
\le \sum_{i=1}^{N-1} e_i + \varepsilon \sum_{i=N}^n s_i
\le \sum_{i=1}^{N-1} e_i + \varepsilon |\guesslang_n|.
\end{equation}
Dividing by $|\guesslang_n|$ gives
\begin{equation}
\frac{E_n}{|\guesslang_n|}
\le
\frac{\sum_{i=1}^{N-1} e_i}{|\guesslang_n|} + \varepsilon.
\end{equation}
Since $\guesslang_n\uparrow\guesslang$ and $\guesslang$ is an infinite set, the first term tends to $0$, so
$\limsup_{n\to\infty} \frac{E_n}{|\guesslang_n|} \le \varepsilon$.
As $\varepsilon>0$ can be chosen arbitrarily, we conclude $\lim_{n\to \infty} \frac{E_n}{|\guesslang_n|}= 0$, and therefore
\begin{equation}
\lim_{n\to \infty} p_n
= \lim_{n\to \infty} 1 - \frac{E_n}{|\guesslang_n|}
\to 1.
\end{equation}
Hence $\setlowerprecision{\targetlang}{\exhaustion{\guesslang}}=\liminf_{n\to\infty} p_n = 1$.
\end{proof}

\newpage
\section{Generation without novelty}\label{app:gen-no-novelty}
\subsection{Generation without novelty but perfect tail precision}\label{app:gen-no-novelty-tail-precision}
\begin{restatable}{theorem}{TheoremValidGenerationWithoutNovelty}\label{thm:gen-valid-without-novelty}
Let $\coll$ be a countable collection of languages and let $\targetlang\in\coll$ be the target language.
Fix an enumeration $\exhaustion{\targetlang}=\{\targetlang_n\}_{n=0}^\infty$ of $\targetlang$.
There exists a generator $\gen$ such that, for any $\func$-bounded adversarial exhaustion
$\exhaustion{\advlang}=\{\advlang_n\}_{n=0}^\infty$ with $\advlang_n\uparrow \advlang\subseteq\targetlang$
and $\setlowerrecall{\exhaustion{\targetlang}}{\advlang}\ge \alpha$, the generated exhaustion
$\exhaustion{\guesslang}$ is $\func$-bounded and satisfies:
\begin{equation}
    \setlowerrecall{\exhaustion{\targetlang}}{\guesslang} \geq \alpha,
    \qquad
    \setlowertailprecision{\targetlang}{\exhaustion{\guesslang}} = 1,
    \qquad
    \setlowerprecision{\targetlang}{\exhaustion{\guesslang}} = 1.
\end{equation}
Moreover, for $c$-step exhaustions the recall guarantee is tight in an adversarial sense.
\end{restatable}

\begin{proof}
Define $\gen$ by simply copying the adversary:
\begin{equation}
    \gen(\advlang_{n}, \guesslang_{n-1}) \defeq \Delta\advlang_n \qquad \text{for all } n\ge 0.
\end{equation}
Hence, $\Delta \guesslang_n = \Delta \advlang_n$, and $\guesslang_n = \advlang_n$. Then $\exhaustion{\guesslang}$ is $\func$-bounded whenever $\exhaustion{\advlang}$ is.

\textbf{Precision and tail precision.}
Since $\guesslang_n=\advlang_n\subseteq \targetlang$ for all $n$, we have
\begin{equation}
\setlowerprecision{\targetlang}{\exhaustion{\guesslang}}
= \liminf_{n\to\infty}\frac{|\guesslang_n\cap\targetlang|}{|\guesslang_n|}
= \liminf_{n\to\infty}\frac{|\guesslang_n|}{|\guesslang_n|}
= 1.
\end{equation}
The step-wise tail precision $t_n = 1$ whenever $\Delta \guesslang_n = \emptyset$. 

Furthermore,  $\Delta\guesslang_n=\Delta\advlang_n\subseteq\targetlang$,
so
\begin{equation}
\setlowertailprecision{\targetlang}{\exhaustion{\guesslang}}
= \liminf_{n\to\infty}\frac{|\targetlang\cap \Delta\guesslang_n|}{|\Delta\guesslang_n|}
= 1.
\end{equation}

\textbf{Recall.}
Because $\guesslang_n=\advlang_n$ for all $n$,
\begin{equation}
    \setlowerrecall{\exhaustion{\targetlang}}{\guesslang}
    =
    \setlowerrecall{\exhaustion{\targetlang}}{\advlang}
    \ge \alpha.
\end{equation}

\textbf{Tightness of the recall bound.} We construct a case in which the recall cannot exceed $\alpha$.

Assume $\func$ enforces a $c$-step exhaustion.
Fix $\targetlang\in\coll$ and a strict sublanguage $\advlang\subset \targetlang$ with
$\setlowerrecall{\exhaustion{\targetlang}}{\advlang}=\alpha<1$, and assume also $\advlang\in\coll$.

Consider a $c$-step exhaustion of $\advlang$. Any generator $\gen$ that is required to satisfy $\setlowertailprecision{\targetlang}{\exhaustion{\guesslang}} = 1$ for \emph{every} possible true language in $\coll$
must in particular satisfy it when the true language is $\advlang$.
Recall that for each $n$, we have
\begin{equation}
t_n =
\begin{cases}
1, & \text{if } |\Delta \guesslang_n|=0,\\[2pt]
\frac{|\advlang\cap \Delta\guesslang_n|}{|\Delta\guesslang_n|}, & \text{otherwise.}
\end{cases}
\end{equation}
Since $|\Delta\guesslang_n|\le c$, each $t_n$ belongs to the finite set
\begin{equation}
F_c \defeq \Bigl\{\frac{i}{j} \colon i,j \in \N_0, 1\le j \le c, 0 \le i\le j\Bigr\}.
\end{equation}
If $\liminf_{n\to\infty} t_n = 1$, then we must have $t_n=1$ for all sufficiently large $n$ because $F_c$ is finite:
otherwise some value $<1$ would occur infinitely often, forcing the $\liminf$ to be $\le \max(F_c\setminus\{1\})<1$.
Hence, when the true language is $\advlang$, the generator can output strings outside $\advlang$ only finitely many times.

Now note that the adversary exhaustion can be identical whether the true language is $\advlang$ or $\targetlang$. It follows that for the instance where the target language is $\targetlang$,
the generator must also eventually restrict to $\advlang$, as it cannot possibly know if the true language is $\targetlang$ or $\advlang$. Consequently it cannot guarantee asymptotic recall exceeding that of $\advlang$ inside $\targetlang$, i.e., it cannot guarantee
$\setlowerrecall{\exhaustion{\targetlang}}{\guesslang}>\alpha$ in general.
\end{proof}

\subsection{Generation without novelty and without perfect tail precision}\label{app:gen-no-novelty-no-tail-precision}

\begin{algorithm}[t]
\caption{Language generation in the limit without novelty or perfect tail precision. \textbf{Input:} adversarial exhaustion $\{\advlang_n\}_{n= 0}^\infty$ with $\advlang_n \uparrow \advlang \subseteq \targetlang$, enumeration $\exhaustion{E}$ of $\kleene{\alphabet}$, infinite exploration set $\exploration \subset \N$. \textbf{Ensures:} increasing exhaustion $\exhaustion{\guesslang}=\{\guesslang_n\}_{n= 0}^\infty$.}
\label{alg:gen-without-novelty}
\begin{algorithmic}[1]
\STATE $\guesslang_0 \leftarrow \emptyset$, $\guesslang_1 \leftarrow \advlang_1$, $j \leftarrow 1$
\STATE $\explorationidx_{\text{next}} \leftarrow \min(\exploration)$

\FOR{$n = 1,2,3,\dots$}
    \IF{$|\guesslang_n| \geq \explorationidx_{\text{next }} \AND \Delta\advlang_{n+1}\neq\emptyset$} 
   
        \STATE $\guesslang_{n+1} \leftarrow \guesslang_{n} \cup \Delta E_j$ \LineComment{exploration round}
        \STATE $j \leftarrow j+1$
        \STATE $\explorationidx_{\text{next}} \leftarrow \min\{\explorationidx \in \exploration : \explorationidx > \explorationidx_{\text{next}}\}$

    \ELSE
    \STATE $\guesslang_{n+1} \leftarrow \guesslang_n \cup \Delta \advlang_{n+1}$ \LineComment{safe round}
    \ENDIF
\ENDFOR
\end{algorithmic}
\end{algorithm}

\TheoremGenerationWithoutNovelty*
\begin{proof}
Since novelty is not required, the generator may repeat the adversary's outputs. We will construct an algorithm that mimics the adversary's strings and occasionally introduces new, potentially incorrect strings at a vanishing rate. We call these insertions exploration rounds. Fix an arbitrary enumeration $\exhaustion{E}=\{E_n\}_{n=0}^\infty$ of $\kleene{\alphabet}$. Choose an exploration set $\exploration\subset\N$ according to \Cref{def:exploration-set}, and let $\explorationidx(m)\defeq |\exploration\cap\{1,\dots,m\}|$. We use $\exploration$ to trigger exploration rounds whenever $|\guesslang_n|$ reaches an $\explorationidx \in \exploration$. We construct an $\func$-bounded $\exhaustion{\guesslang}=\{\guesslang_n\}_{n= 0}^\infty$ iteratively as described in \cref{alg:gen-without-novelty}. The exhaustion $\exhaustion{\guesslang}$ is $\func$-bounded because, on a safe round, it adds exactly the strings in $\Delta\advlang_{n+1}$, while an exploration round adds only one string and is taken only if $\Delta\advlang_{n+1}\neq\emptyset$. In the latter case, $\func(n+1)\geq1$, and hence the exploration output also respects the bound.

By construction, $\guesslang_n$ is increasing. Moreover, we ensure that exploration happens only whenever $|\guesslang_n|$ reaches the next value in $\exploration$. This will ensure that for $\guesslang_n$ with $|\guesslang_n| = m_n$, the number of explorative strings is at most $\explorationidx(m_n)$.

\paragraph{Recall.}
Every $\Delta E_n$ gets added eventually, hence $\guesslang_n \uparrow \kleene{\alphabet}$ and 
\begin{equation}
    \setlowerrecall{\exhaustion{\targetlang}}{\guesslang}
= \liminf_{n\to\infty}\frac{|\targetlang_n\cap \guesslang|}{|\targetlang_n|}
= \liminf_{n\to\infty}\frac{|\targetlang_n|}{|\targetlang_n|}
= 1.
\end{equation}

\paragraph{Precision.}
For each $\guesslang_n$ with $|\guesslang_n| = m_n$ the number of outputs from explorations (and hence possible hallucinations) is at most $\explorationidx(m_n)$; all other outputs
are parroted adversary elements and thus lie in $\targetlang$ because $\advlang\subseteq\targetlang$.
Hence
\begin{equation}
    |\guesslang_n\cap\targetlang| \ge m_n - \explorationidx(m_n)
\end{equation}
and therefore
\begin{align}
    \setlowerprecision{\targetlang}{\exhaustion{\guesslang}}
= \liminf_{n\to\infty}\frac{|\guesslang_n\cap\targetlang|}{|\guesslang_n|}
\\\ge \liminf_{n\to\infty}\left(1-\frac{\explorationidx(m_n)}{m_n}\right)
= 1.
\end{align}
\end{proof}

\newpage
\section{Pods algorithm}
In this section, we give \citeposs{kleinberg_language_2025} algorithm and describe its main ideas. We fix a countable collection of languages $\coll \subseteq \mathcal{P}(\kleene{\alphabet})$
together with
(i) a global order on languages in $\coll$, yielding an ordered sequence $(\langs'_i)_{i\ge1}$ and
(ii) an enumeration of $\kleene{\alphabet}$ that induces a total order $\prec$. For a string $\strx\in \targetlang$ we define its successor by $\mathrm{succ}_\targetlang(\strx)$ in $(\targetlang,\prec)$.\footnote{Note that $\prec$ can be applied to any language $\langs$ in $\kleene{\alphabet}$, as $\langs \subseteq \kleene{\alphabet}$.}

We recall the two main technical primitives from \citet{kleinberg_language_2025}:
the \emph{identified intersection} procedure and the \emph{pods} construction.

\paragraph{Consistency and intersection chains.}
Let $\targetlang\in\coll$ be the target language.
An adversary reveals an exhaustion $\exhaustion{\advlang}$. Note that in \citet{kleinberg_language_2025} this exhaustion is single-step-bounded. However, we write \Cref{alg:pods-generator} in a way such that it can also be run with a $k$-batched adversary, a setting that we analyze in \Cref{app:lowerrecall-batched-adv}.

\begin{definition}
    Let $\exhaustion{\advlang}=\{\advlang_n\}_{n=0}^\infty$ be an adversarial exhaustion and $\langs$ an arbitrary language. We call $\langs$ \defn{consistent} with $\exhaustion{\advlang}$ up to timestep $n$ if $\advlang_n \subseteq \langs$.
\end{definition}

\begin{definition}
Let the consistent languages at time $n$ be listed in the global order as $\langs^{(n)}_1, \langs^{(n)}_2, \langs^{(n)}_3,\dots$, where $(\langs^{(n)}_i)_{i\ge 1}$ is obtained by scanning $\langs_1,\langs_2,\dots$ and keeping exactly those that contain $\advlang_n$. If only finitely many, say $k$, consistent languages remain at time $n$, we pad the sequence by repeating the last language, namely
\begin{equation}
    \langs_l^{(n)} = \langs_k^{(n)}, \qquad l \geq k
\end{equation}

We define the descending \defn{chain of intersections} at time $n$ by
\begin{equation}
    \intersect^{(n)}_i \defeq \bigcap_{j\le i} \langs^{(n)}_j,
\qquad
\chain_n \defeq \intersect^{(n)}_1 \supseteq \intersect^{(n)}_2 \supseteq \intersect^{(n)}_3 \supseteq \cdots.
\end{equation}
\end{definition}
Note that $\chain_n$ is always an infinite chain \citep[Eq.~3]{kleinberg_language_2025}. See also \Cref{alg:update-chain}.
\paragraph{Identified intersection.}
\Cref{alg:update-identified-intersection} describes how to find the best guess $\intersect^{(n)}$ of the target language by intersecting the remaining consistent hypotheses. $\intersect^{(n)}$ is called \defn{identified intersection}. The algorithm chooses an index $i(n)\in\N$ such that the output $\intersect^{(n)} \defeq \intersect^{(n)}_{i(n)}$
is an infinite set (except in finitely many cases) and moves in a controlled way over time as $n$ increases. Concretely, if the descending chain of intersections remains the same between rounds $n$ and $n+1$, the algorithm either keeps $i(n+1)=i(n)$ or increases it by one (moving one step deeper in the intersection chain) provided the next intersection is still infinite. If the set of consistent hypotheses changes, the algorithm resets 
$i(n+1)$ to the deepest level at which the old and new intersection chains still agree on an infinite set.

We use the following terminology matching \citealp{kleinberg_language_2025}.
We say that $\intersect^{(n)}$ is \defn{valid} if $\intersect^{(n)}\subseteq \targetlang$; and \defn{full} if $\advlang \subseteq \intersect^{(n)}$.

\citet[Lem.~2.5]{kleinberg_language_2025} establish the following structural properties of $\intersect^{(n)}$:
there exists a finite time $N_0$ such that for all $n\ge N_0$,
\begin{enumerate}
    \item (\emph{eventual validity}) $\intersect^{(n)}$ is valid;
    \item (\emph{eventual comparability}) either $\intersect^{(n)}\subset \intersect^{(n+1)}$ or $\intersect^{(n+1)}\subset\intersect^{(n)}$;
    \item (\emph{full infinitely many times}) the event \emph{$\intersect^{(n)}$ is full} occurs infinitely often.
\end{enumerate}
Intuitively, $\intersect^{(n)}$ is the algorithm’s current safe proxy for $\targetlang$. Once $\intersect^{(n)}$ is valid, every string in it lies in $\targetlang$, so outputting from $\intersect^{(n)}$ preserves validity. Eventual comparability means that after finitely many steps the sequence evolves keeping some inclusion relation between consecutive identified intersections, which the later analysis repeatedly exploits to control how the identified intersection changes over time. Finally, the fact that $\intersect^{n}$ is full infinitely often guarantees that the process repeatedly returns to an intersection that contains $\advlang$, a property which can be exploited to support high recall.

\paragraph{Pods generator algorithm.} \citeauthor{kleinberg_language_2025} introduce a generating procedure (see \cref{alg:pods-generator}) that achieves the properties in \Cref{thm:half-lower-bound}. In each timestep $n$ it decides on an \defn{aggressive set} $\smash{\widetilde{\intersect}}^{(n)}$ (see \cref{alg:aggressive-set}, as defined in \citet[Sec.~3.2]{kleinberg_language_2025}), which consists of strings that could be part of $\targetlang$, though including them may risk overshooting. Intuitively, $\smash{\widetilde{\intersect}}^{(n)}$ is either the current identified intersection $\intersect^{(n)}$ or a controlled fallback superset from earlier rounds. Furthermore, in each round, a \defn{pod} $\Pod_n$ is determined. A pod $\Pod_n$ is a set of strings that have priority to be output. The pod sizes are fixed as a growing sequence, for example $(s_n)_{n\ge 1}$ with $s_n=n$. The algorithm maintains a set $\unavailableset_n$ of  strings that are \emph{used}, meaning they have been revealed by the adversary, already output,
or already placed into previous pods:
\begin{equation}
    \unavailableset_n \defeq \advlang_n \cup\guesslang_{n-1} \cup \bigcup_{\nu<n} \Pod_\nu.
\end{equation}
The \defn{pod} $\Pod_n$ is formed by the $s_n$ $\prec$-smallest elements of $\smash{\widetilde{\intersect}}^{(n)}\setminus\unavailableset_n$. Furthermore, the algorithm maintains a \defn{pod pool} $\podpool_n$ of available pod elements waiting to be output. The algorithm outputs the smallest element in $\podpool_n\setminus (\advlang_n \cup \guesslang_{n-1})$. By construction, the generated exhaustion $\exhaustion{\guesslang}$ satisfies the novelty constraint.

\begin{algorithm}[t]
\caption{Pods algorithm for language generation in the limit  (after \citealp{kleinberg_language_2025}). \textsc{UpdateChain} takes the adversary's input and a sequence of languages and returns an updated sequence of languages and a chain of intersections of consistent languages; \textsc{IdIntersect} takes the timestep and a level index, two intersection chains and an identified intersection and returns the next level index and the next identified intersection; \textsc{AggrSet} takes two identified intersections and two intersection chains and returns an aggressive set, i.e., a superset of the current identified intersection. \textbf{Input:} adversarial exhaustion $\{\advlang_n\}_{n= 0}^\infty$ with $\advlang_n \uparrow \advlang \subseteq \targetlang$, pod size sequence $(s_n)_{n\geq 1}$, ordered sequence $(\langs'_i)_{i\ge1}$ of languages in $\coll$ \textbf{Ensures:} single-step exhaustion $\exhaustion{\guesslang}=\{\guesslang_n\}_{n= 0}^\infty$.}
\label{alg:pods-generator}
\begin{algorithmic}[1]
\STATE $\guesslang_0 \leftarrow \emptyset$, $\unavailableset_0 \leftarrow \emptyset$ \COMMENT {unavailable strings}, $\podpool_0 \leftarrow \emptyset$ \COMMENT{pod pool}, $(\langs^{(0)}_i)_{i\ge1} \leftarrow (\langs'_i)_{i\ge1}$ \COMMENT {consistent languages}, $k_0 \leftarrow 1$, $\intersect^{(0)}\gets\emptyset$, $\chain_0\gets\emptyset$ 

\FOR{$n = 1,2,3,\dots$}
    \STATE Observe $\advlang_n$
    \STATE $\bigl((\langs^{(n)}_i)_{i\ge1}, \chain_n\bigr) \gets \textsc{UpdateChain}\bigl(\Delta\advlang_{n}, (\langs^{(n-1)}_i)_{i\ge1}\bigr)$
        \STATE $(\intersect^{(n)}, k_n) \gets$ \textsc{IdIntersect}$(n, k_{n-1}, \intersect^{(n-1)}, \chain_{n-1}, \chain_n)$
        \STATE $\widetilde{\intersect}^{(n)} \gets$\textsc{AggrSet}$(n,\intersect^{(n-1)},\intersect^{(n)}, \chain_{n-1}, \chain_n)$
        \STATE $\unavailableset_n \gets \unavailableset_{n-1} \cup \advlang_n$. 
        \STATE $\Pod_n\gets$ the $s_n$ $\prec$-smallest elements of $\widetilde{\intersect}^{(n)}\setminus\unavailableset_n$.
        \STATE $\unavailableset_n \gets \unavailableset_n \cup \Pod_n$
        \STATE $\podpool_n \gets \podpool_{n-1} \cup \Pod_n$.
        \STATE $\guesslang_n \gets \guesslang_{n-1} \cup \min_{\prec}(\podpool_n\setminus (\advlang_n \cup \guesslang_{n-1}))$.
        \STATE $\unavailableset_n \gets \unavailableset_n \cup \guesslang_n$.
\ENDFOR
\end{algorithmic}
\end{algorithm}

\begin{algorithm}[t]
\caption{\textsc{UpdateChain}$(\Delta\advlang_{n},\,(\langs^{(n-1)}_i)_{i\ge 1})$, \textbf{Input:} current adversarial increment $\Delta \advlang = \advlang_n \setminus \advlang_{n-1}$,
previous consistent list $(\langs^{(n-1)}_i)_{i\ge 1}$ \textbf{Ensures:} updated consistent list $(\langs^{(n)}_i)_{i\ge 1}$, intersection chain
$\chain_n=(\intersect^{(n)}_i)_{i\ge 1}$}
\label{alg:update-chain}
\begin{algorithmic}[1]
\STATE $t \gets 0$
\FOR{$i = 1,2,3,\dots$}
    \IF{$\Delta \advlang_n \subseteq \langs^{(n-1)}_i$}
        \STATE $t \gets t + 1$
        \STATE $\langs^{(n)}_t \gets \langs^{(n-1)}_i$
        \IF{$t = 1$}
            \STATE $\intersect^{(n)}_1 \gets \langs^{(n)}_1$
        \ELSE
            \STATE $\intersect^{(n)}_t \gets \intersect^{(n)}_{t-1} \cap \langs^{(n)}_t$
        \ENDIF
    \ENDIF
\ENDFOR
\STATE \textbf{return} $\bigl((\langs^{(n)}_i)_{i\ge 1},\, (\intersect^{(n)}_i)_{i\ge 1}\bigr)$
\end{algorithmic}
\end{algorithm}

\begin{algorithm}[t]
\caption{\textsc{IdIntersect}$(n,k_{n-1},\, \intersect^{(n-1)},\, \chain_{n-1},\, \chain_{n})$ \textbf{Input:} timestep $n$, identified intersection $\intersect^{(n-1)}$, level index $k_{n-1}$, descending chains $\chain_{n-1}$ and $\chain_{n}$ \textbf{Ensures:}  updated identified intersection $\intersect^{(n)}$ and updated level index $k_n$.}
\label{alg:update-identified-intersection}
\begin{algorithmic}[1]

\IF{$n=1$}
    \STATE \textbf{return} $(\intersect^{(1)}_1,\, 1)$
\ENDIF

\STATE $k \leftarrow k_{n-1}$

\IF{$\chain_{n-1} = \chain_n$}
    \IF{$|\intersect^{(n-1)}_{k+1}| = \infty$}
        \STATE \textbf{return} $(\intersect^{(n)} \leftarrow \intersect^{(n-1)}_{k+1},\, k_n \leftarrow k+1)$ \COMMENT{move down one level}
    \ELSE
        \STATE \textbf{return} $(\intersect^{(n)} \leftarrow \intersect^{(n-1)}_{k},\, k_n \leftarrow k)$ \COMMENT{stay if next set becomes finite}
    \ENDIF
\ELSE
    \STATE $k^\star \leftarrow 0$
    \FOR{$j = 1,2,3,\dots$}
        \IF{$\intersect^{(n-1)}_j = \intersect^{(n)}_j$ \AND $|\intersect^{(n-1)}_j|=\infty$}
            \STATE $k^\star \leftarrow j$
        \ELSE
            \STATE \textbf{break}
        \ENDIF
    \ENDFOR

    \IF{$k^\star \ge 1$}
        \STATE \textbf{return} $(\intersect^{(n)} \leftarrow \intersect^{(n-1)}_{k^\star},\, k_n \leftarrow k^\star)$ \COMMENT{largest common infinite prefix}
    \ELSE
        \STATE \textbf{return} $(\intersect^{(n)} \leftarrow \intersect^{(n)}_{1},\, k_n \leftarrow 1)$ \COMMENT{no common infinite prefix}
    \ENDIF
\ENDIF
\end{algorithmic}
\end{algorithm}

\begin{algorithm}[t]
\caption{\textsc{AggrSet}$(n, \intersect^{(n-1)}, \intersect^{(n)}, \chain_{n-1}, \chain_n)$ \textbf{Input:} round $n$, identified intersection $\intersect^{(n-1)}$ and $\intersect^{(n)}$, descending chains $\chain_{n-1}$ and $\chain_{n}$ \textbf{Ensures:}  Aggressive set $\smash{\widetilde{\intersect}}^{(n)}$.}
\label{alg:aggressive-set}
\begin{algorithmic}[1]
\IF{$n=1$}
    \STATE \textbf{return} $\intersect^{(n)}$
\ENDIF

\IF{$\intersect^{(n)}=\intersect^{(n-1)}$}
    \STATE \textbf{return} $\intersect^{(n)}$ \COMMENT{no fallback}
\ELSIF{$\intersect^{(n)}\subsetneq \intersect^{(n-1)}$}
    \STATE \textbf{return} $\intersect^{(n-1)}$ \COMMENT{aggressively guess the previous larger set}
\ELSIF{$\intersect^{(n)}\supsetneq \intersect^{(n-1)}$}
    \STATE \textbf{return} $\intersect^{(n)}$ \COMMENT{aggressively guess the new larger set}
\ELSE
    \STATE \COMMENT{incomparable case: search for a common predecessor in both chains}
    \STATE $\mathcal{S}_{n-1} \leftarrow \{J\in \chain_{n-1} \colon J \supseteq \intersect^{(n-1)}\}$
    \STATE $\mathcal{S}_{n} \leftarrow \{J\in \chain_{n} : J \supseteq \intersect^{(n)}\}$
    \STATE $\mathcal{S} \leftarrow \mathcal{S}_{n-1}\cap \mathcal{S}_{n}$
    \IF{$\mathcal{S}=\emptyset$}
        \STATE \textbf{return} $\intersect^{(n)}$ \COMMENT{no fallback possible}
    \ELSE
        \STATE \textbf{return} a $\subseteq$-minimal element of $\mathcal{S}$
    \ENDIF
\ENDIF
\end{algorithmic}
\end{algorithm}

\newpage

\section{Recall bounds in a batched adversary setting}\label{app:batched-adv}

\subsection{Lower recall bound}\label{app:lowerrecall-batched-adv}

\LemmaBatchedLower*

\begin{proof}
We use \citeposs{kleinberg_language_2025} algorithm as described in \cref{alg:pods-generator}. By construction, $\exhaustion{\guesslang}$ satisfies the novelty constraint, as all outputs come from pod pools $\podpool_n$, that contain only unseen strings at step $n$. The major difference to the original setting is that in each timestep the adversary reveals $k$ strings, and hence we choose the consistent languages based on $k$ strings instead of just one. The procedure of computing the intersections remains the same. Note that the properties of eventual validity, eventual comparability, and being full infinitely often transfer to this setting. The proof is analogous to the proof of \citet[Lemma 2.5]{kleinberg_language_2025}. The procedure of choosing the aggressive sets and pods remains also the same. Hence, our proof follows the structure of \citet[Thm.~3.5]{kleinberg_language_2025} closely and differs only in certain counting steps.

\paragraph{Tail precision and precision.} After finitely many timesteps $N$ the identified intersection $\intersect^{(n)}$ and the aggressive set $\smash{\widetilde{\intersect}}^{(n)}$ remain valid (see \citet[Lemma 2.5]{kleinberg_language_2025} and \Cref{alg:aggressive-set}). Hence, up to that point only finitely many wrong strings can be added to $\podpool_N$ and afterwards only strings from $\targetlang$ are added. This immediately implies 
\begin{equation}
    t_n = \frac{|\Delta \guesslang_n \cap \targetlang|}{|\Delta \guesslang_n|} = 1 \qquad n \geq N
\end{equation}
and thus $\setlowertailprecision{\targetlang}{\exhaustion{\guesslang}} = 1$ and $\setlowerprecision{\targetlang}{\exhaustion{\guesslang}} = 1$ by \Cref{lem:tail-precision-implies-precision}.

\paragraph{Recall.} Denote by $\podpool = \bigcup_{n \geq 1} \Pod_n$ the set of all strings that ever appear in a pod. We define a partition into a good and bad set as in \citet[Thm.~3.5]{kleinberg_language_2025}:
\begin{align}
    &B_g := \{\strx\in \advlang\setminus(\podpool\cup \guesslang):\ \forall \str\in \Pod_{\nu(\strx)-1},\ \str\preceq \mathrm{succ}_\targetlang(\strx)\},
\\&B_b := (\advlang\setminus(\podpool\cup \guesslang))\setminus B_g,
\end{align}
where $\nu(\strx)$ is the first round in which $\strx$ appears among $\advlang_{\nu(\strx)}$.

Let $N_1$ be the time from which on all identified intersections are valid. We let $N$ be sufficiently large such that at $N$ we have already observed a full intersection twice since $N_1$.

\paragraph{Bounding $B_b$ and $B_g$.} We fix an $s$. By Lemma~\ref{lem:batched-rho}, there exists a map $\rho$ from bad strings to pod indices such that
$\Pod_{\rho(\strx)}$ contains at least $s$ elements, all $\prec \strx$, and each index has at most $k$ preimages.
This implies
\begin{equation}
|B_b\cap\targetlang_n|
\ \le k\big(
\frac{1}{s}\,|\podpool\cap\targetlang_n|\ +\bigo(N) +\bigo(s)\big).
\label{eq:bb}
\end{equation}

Furthermore, we can bound $|B_g \cap \targetlang_n|$ by mapping each $\strx \in B_g$ to the pod output at time $\nu(\strx) - 1$. Hence, at most $k$ strings get mapped to the same pod. For a string $\strx$ with $\nu(\strx) \geq s$, each previous pod contains at least $s$ elements, which are $\prec \strx$ by definition of $B_g$. Hence,
\begin{equation}
|B_g \cap \targetlang_n| \le k\big(\frac{1}{s}  |\podpool\cap\targetlang_n| + \bigo(N) + \bigo(s)\big)
\label{eq:bg}
\end{equation}

\paragraph{Charging pod elements to outputs.}
The only further major difference appears in \citeposs{kleinberg_language_2025} charging argument (their Eq.~(14)).
In our $k$-batched setting, at most $k$ new adversary elements can enter the pool per round, hence
\begin{equation}
|(\podpool\cap(\advlang\setminus \guesslang))\cap\targetlang_n|
\ \le\
k\,|\podpool \cap \guesslang\cap \targetlang_n| = k\,|\guesslang\cap \targetlang_n|,
\label{eq:ch}
\end{equation}
as $\guesslang \subseteq \podpool$.
\paragraph{Combining bounds.}
Every $\strx\in\advlang$ lies in exactly one of
\begin{equation}
(\podpool\cap(\advlang\setminus \guesslang)),\qquad (\podpool\cap \guesslang)\cap \advlang,\qquad B_b,\qquad B_g.
\end{equation}
Intersecting with $\targetlang_n$ and taking cardinalities yields
\begin{align*}
|\advlang\cap\targetlang_n|
&=
|(\advlang\cap\guesslang)\cap\targetlang_n|
+
|(\podpool\cap(\advlang\setminus\guesslang))\cap\targetlang_n|
+
|B_b\cap\targetlang_n|
+
|B_g\cap\targetlang_n|.
\end{align*}
The first term is at most $|\guesslang\cap\targetlang_n|$.
The second term is bounded by $k|\guesslang\cap\targetlang_n|$ by \Cref{eq:ch}.
The last two terms are bounded using \Cref{eq:bb} and \Cref{eq:bg}.
Altogether,
\begin{equation}
|\advlang\cap\targetlang_n|
\ \le\
(k+1)|\guesslang\cap \targetlang_n|+\frac{2k}{s}|\podpool\cap\targetlang_n|+\bigo(kN) + \bigo(ks)
\label{eq:C-upper}
\end{equation}

\paragraph{Bounding pool mass by outputs.}
We next bound $|\podpool\cap\targetlang_n|$ in terms of $|\guesslang \cap \targetlang_n|$.
Using that $\guesslang\subseteq \podpool$ (the algorithm always outputs from the pool),
we may decompose
\begin{equation}
\podpool\cap\targetlang_n
\subseteq
(\guesslang\cap\targetlang_n)\ \cup\ ((\podpool\cap(\advlang\setminus\guesslang))\cap\targetlang_n).
\end{equation}
Taking cardinalities and applying \Cref{eq:ch} gives
\begin{equation}
|\podpool\cap\targetlang_n|
\ \le\
|\guesslang\cap\targetlang_n|
+
|(\podpool\cap(\advlang\setminus\guesslang))\cap\targetlang_n|
\ \le\
(k+1)|\guesslang\cap \targetlang_n|.
\label{eq:P-upper}
\end{equation}

Substituting \Cref{eq:P-upper} into \Cref{eq:C-upper} yields
\begin{equation}
    |\advlang\cap\targetlang_n|
\ \le\
\Bigl((k+1)+\frac{2k(k+1)}{s}\Bigr)|\guesslang \cap \targetlang_n| + \bigo(kN + ks).
\end{equation}
Dividing by $n$ and taking $\liminf_{n\to\infty}$, we obtain
\begin{equation}
    \liminf_{n\to \infty} \frac{|\advlang \cap \targetlang_n|}{n} \le\
\Bigl((k+1)+\frac{2k(k+1)}{s}\Bigr)\liminf_{n \to \infty} \frac{|\guesslang \cap \targetlang_n|}{n}.
\end{equation}

Letting $s\to\infty$ gives
\begin{equation}
    \setlowerrecall{\exhaustion{\targetlang}}{\guesslang} \geq \frac{1}{k+1} \setlowerrecall{\exhaustion{\targetlang}}{\advlang}
\end{equation}
which proves the claim.
\end{proof}

\begin{lemma}[Batched bad-string pod map]
\label{lem:batched-rho}
Fix $s\in\N$ and consider the pods generator algorithm (Algorithm~\ref{alg:pods-generator})
in the $k$-batched setting.
There exists a finite time $N$ and a map
\begin{equation}
    \rho:\ \{\strx\in B_b : \nu(\strx)>N\}\to \N
\end{equation}
such that for every $\strx\in B_b$ with $\nu(\strx)>N$:
\begin{enumerate}
    \item all elements of $\Pod_{\rho(\strx)}$ are $\prec \strx$ and $|\Pod_{\rho(\strx)}|\ge s$;
    \item (\emph{bounded multiplicity}) each index $n'\in\N$ has at most $k$ preimages:
    $|\rho^{-1}(n')|\le k$.
\end{enumerate}
\end{lemma}

\begin{proof}
The proof follows \citet[Lem.~3.6]{kleinberg_language_2025}. Let $N_1$ be the time from which on all identified intersections are valid. We let $N$ be sufficiently large such that at $N$ we have already observed a full intersection twice since $N_1$. Furthermore, we assume $n \geq s$ for the rest of the proof. Fix $\strx\in B_b$ with appearance time $\nu(\strx)$, yielding $\strx \in \intersect^{(\nu(\strx))}$. $\strx\in B_b$ means that $\strx \notin \podpool \cup \guesslang$ and hence in no previous pod, however, a pod element beyond $\mathrm{succ}_\targetlang(\strx)$ is in $\Pod_{\nu(\strx)-1}$. Hence, $\strx\notin \intersect^{(\nu(\strx)-1)}$ and $\intersect^{(\nu(\strx)-1)}$ is inconsistent
at round $\nu(\strx)$ and the identified intersection must make an upward move to a \emph{superset} within the prior chain,
exactly as in \citet[Lem.~3.6]{kleinberg_language_2025}. Consequently, there exists a time $n<\nu(\strx)$ such that the identified intersections at times $n$ and $\nu(\strx)$
coincide, and we define $\rho(\strx)\coloneqq n$. If there are multiple previous times at which the identified intersections coincide with the one at $\nu(\strx)$, we choose the largest such index $n$.

Consider the pod $\Pod_{n}$ created at time $n$.
By construction, $\Pod_{n}$ contains the $s_{n}$ smallest unused strings in the aggressive set, which is a superset of the corresponding identified
intersection. Because $\strx$ first appears at time $\nu(\strx)$ and $\strx\notin \podpool\cup \guesslang$, it was unused at time $n$; however, it appeared in the identified intersection and the aggressive set at that time. Thus every element added to $\Pod_{n}$ must be $\prec \strx$, and since $n\ge s$ for all sufficiently
large times, we also have $|\Pod_{n}|=s_{n}\ge s$.

Finally, the only difference from \citet[Lem.~3.6]{kleinberg_language_2025} is that $\rho$
is no longer injective: in the $k$-batched model, at most $k$ new strings can arrive in a single round, possibly triggering one upward move together. Each string $\strx \in B_b$ from that round yields the same $\rho(\strx)$. However, strings arriving in later rounds will be mapped to another index. Hence, $|\rho^{-1}(n)|\le k$ for all $n$.
\end{proof}

\subsection{Upper recall bound}\label{app:upperrecall-batched-adv}
\LemmaBatchedAdversaryUpper*

\begin{proof}
Assume any countable collection of languages $\coll$. Assume a single-step novel generator $\gen$ that satisfies $\setlowertailprecision{\targetlang}{\exhaustion{\guesslang}}= 1$. Let $\exhaustion{\targetlang}=\{\targetlang_n\}_{n=0}^\infty$ be an enumeration of $\targetlang$.
\paragraph{Full exhaustion.} Assume first that the adversary fully reveals $\targetlang$, i.e., $\advlang_n\uparrow \targetlang$ and hence $\setlowerrecall{\exhaustion{\targetlang}}{\advlang}=1$.
We define an adversary exhaustion that forces
\begin{equation}\label{eq:full-case}
    \setlowerrecall{\exhaustion{\targetlang}}{\guesslang}\le \frac{1}{k+1}.
\end{equation}

Let $\exploration\defeq\{2^m:m\in\N_0\}$ be the set of so-called repetition steps and define $\explorationidx(n)\defeq|\exploration\cap\{1,\dots,n\}|$. At each timestep $n\in\N$, the adversary reveals a set of at most $k$ new strings as follows. If $n\notin \exploration$, it reveals the $k$ smallest strings in $\targetlang$ (with respect to the ordering induced by the enumeration $\exhaustion{\targetlang}$)
    that have not yet been used by either player, or more formally, it outputs the unused strings from $\targetlang_{i_n}$, where $i_n$ is the smallest index such that $|\targetlang_{i_n} \setminus (\advlang_{n-1} \cup \guesslang_{n-1})| = k$ . If $n\in \exploration$, it reveals the $(k-1)$ smallest such unused strings and, in addition,
    the smallest string in $\targetlang$ that has previously been output by $\gen$ but has not yet been revealed by the adversary (if such a string exists).
This exhaustion is $\func$-bounded with $\func(n)=k$ and satisfies $|\Delta \advlang_n|\le k$ by construction. Moreover, by construction, the adversary outputs all strings in $\guesslang \cap \targetlang$, and hence $\advlang_n\uparrow\targetlang$.

Now consider the first $n$ interaction steps. The generator outputs at most $n$ strings in total.
On each non-repetition step $n' \notin \exploration$ with $n'\leq n$ the adversary reveals the $k$ smallest unused strings, so apart from the $\explorationidx(n)$
repetition steps, each block of $k+1$ newly used strings can contain at most one generator output.
Consequently, among the first $(k+1)n$ strings in $\targetlang_{(k+1)n}$, the generator can cover at most $n + \explorationidx(n)$ of them:
\begin{equation}\label{eq:counting}
    \bigl|\guesslang \cap \targetlang_{(k+1)n}\bigr| \le n + \explorationidx(n).
\end{equation}
Divide by $(k+1)n$ and take $\liminf_{n\to\infty}$. Since $r(n)=\bigo(\log n)$ for powers of two, we have $\explorationidx(n)/n\to 0$, and thus
\begin{equation}
    \liminf_{n\to\infty}\frac{|\guesslang\cap \targetlang_n|}{|\targetlang_n|}
\le
\frac{1}{k+1},
\end{equation}
which is exactly \Cref{eq:full-case}.

\paragraph{Partial exhaustion.}
Now let $\advlang\subseteq\targetlang$ be arbitrary, and let the adversary reveal only strings from $\advlang$ using the same strategy as above,
with smallest string interpreted with respect to the order induced by $\exhaustion{\targetlang}$ restricted to $\advlang$.
Consider a collection $\coll$ that contains both $\advlang$ and $\targetlang$. Crucially, the revealed exhaustion is compatible with either choice of true language in $\{\advlang,\targetlang\}$. Since the generator's behavior depends only on the revealed history, it produces the same output stream in both cases. Therefore, because the generator must eventually reach tail precision $1$ when the true language is $\advlang$,
it can output only finitely many strings outside $\advlang$.
Hence, $\bigl|\guesslang \cap (\targetlang_n\setminus \advlang)\bigr| = \bigo(1).$
We apply the same counting argument as in the full case:
\begin{equation}\label{eq:pointwise}
    |\guesslang\cap \targetlang_n|=|\guesslang\cap (\advlang\cap \targetlang_n)| + \bigo(1) \le \frac{|\advlang\cap \targetlang_n|}{k+1} + \bigo(\log |\advlang\cap \targetlang_n|)
\end{equation}
Dividing by $|\targetlang_n|$ and taking $\liminf_{n\to\infty}$, the second term vanishes and thus
\begin{equation}
\setlowerrecall{\exhaustion{\targetlang}}{\guesslang}
=
\liminf_{n\to\infty}\frac{|\guesslang\cap \targetlang_n|}{|\targetlang_n|}\le \frac{1}{k+1}\liminf_{n\to\infty}\frac{|\advlang \cap \targetlang_n|}{|\targetlang_n|}
=
\frac{1}{k+1}\,\setlowerrecall{\exhaustion{\targetlang}}{\advlang}.  
\end{equation}
\end{proof}
\newpage

\section{Generation with novelty}%

\subsection{Generation with novelty but without perfect tail precision}\label{app:gen-with-novelty}

\TheoremGenerationWithNovelty*

\begin{proof}
Fix an arbitrary enumeration $\exhaustion{E}=\{E_n\}_{n=1}^\infty$ of $\kleene{\alphabet}$. Choose an exploration set $\exploration\subset\N$ according to \Cref{def:exploration-set} and let $\explorationidx(m)\defeq |\exploration\cap\{1,\dots,m\}|$. We will schedule exploration rounds on each $\explorationidx \in \exploration$. We construct $\exhaustion{\guesslang}=\{\guesslang_n\}_{n= 0}^\infty$ iteratively as described in \cref{alg:gen-with-novelty}. By construction, $\guesslang_n$ increases by $1$ in each round, as the outputs are novel. Moreover, exploration happens only whenever the current timestep reaches the next value in $\exploration$. This will ensure that for $\guesslang_n$ with $|\guesslang_n| = n$, the number of exploratory strings is at most $\explorationidx(n)$. We first prove the recall bound.

\paragraph{The $1-\setupperrecall{\exhaustion{\targetlang}}{\advlang}$ bound.}
We show that Algorithm~\ref{alg:gen-with-novelty} generates every string in $\targetlang\setminus\advlang$.
Fix any $\strx\in\targetlang\setminus\advlang$, and let $l$ be such that $\Delta E_l=\{\strx\}$.
Since $\strx\notin\advlang$, the adversary never outputs $\strx$, so $\strx\notin \advlang_m$ at all times $m$.
On exploration rounds, the algorithm outputs the first $\Delta E_j$ not in $\guesslang_{m-1}\cup \advlang_m$ and advances $j$. Exploration rounds occur infinitely often and $\strx$ cannot be made unavailable by $\advlang$, as it does not occur in any $\advlang_m$. Hence, $\strx$ is either produced in a safe round or, if not, the pointer $j$ reaches $l$ eventually, and the algorithm outputs $\strx$ in an exploration round.
Thus $\guesslang$ contains $\targetlang\setminus\advlang$.

Applying Lemma~\ref{lem:upperdensitynoveltyalgo} yields
\begin{equation}
\setlowerrecall{\exhaustion{\targetlang}}{\guesslang}\ \ge\ 1-\setupperrecall{\exhaustion{\targetlang}}{\advlang}.
\end{equation}

\paragraph{The $\frac{\setlowerrecall{\exhaustion{\targetlang}}{\advlang}}{3}$ bound.}
Consider only the safe rounds. Since
$\exploration$ contains no consecutive integers, at most
two adversarial strings are revealed between consecutive
safe rounds. Thus, the safe subroutine sees a $2$-batched
adversary.

Let $\podpool$ be the union of all pods created during
safe rounds. Note that $\guesslang$ contain both safe and
exploration outputs. We now closely follow the charging argument from
\Cref{lem:alpha-over-3} and adapt it to this setting.

First, every element of $\podpool$ is eventually either
generated or revealed by the adversary. Consequently,
\begin{equation}
    \podpool \subseteq \advlang \cup \guesslang.
\end{equation}

Next, consider an element
$\strx\in\podpool\cap(\advlang\setminus\guesslang)$.
At the safe round immediately preceding the first safe
round at which $\strx$ is observed, $\strx$ was still available
in the pod pool. Hence, the safe output at that round
precedes $\strx$ under $\prec$. Since at most two adversarial
strings are newly observed per safe round, each safe
output receives at most two such charges. Using the
common ordering of $\targetlang_n$ and $\prec$, and
discarding finitely many early invalid outputs, we get
\begin{equation}
    |(\podpool\cap(\advlang\setminus\guesslang))
  \cap\targetlang_n|
\le 2|\guesslang\cap\targetlang_n|+O(1).
\end{equation}
Together with $\podpool\subseteq\advlang\cup\guesslang$,
this implies
\begin{equation}
    |\podpool\cap\targetlang_n|
\le 3|\guesslang\cap\targetlang_n|+O(1).
\end{equation}

It remains to bound adversarial strings that are neither
in $\podpool$ nor in $\guesslang$. As in
\Cref{lem:alpha-over-3}, partition these strings into
$B_g$ and $B_b$. The pod-mapping argument continues to
apply: the identified intersection chain evolves only
on safe rounds, with at most two new adversarial strings
per update. Exploration outputs merely remove strings
from the unused set; any string in $B_g\cup B_b$ remains
unused throughout.

Fix $s\ge 2$. Since all sufficiently late pods have
size at least $s$, the same bounded-multiplicity
arguments yield
\begin{equation}
    |B_b\cap\targetlang_n|
\le \frac{2}{s}|\podpool\cap\targetlang_n|+O_s(1)
\end{equation}
and
\begin{equation}
|B_g\cap\targetlang_n|
\le \frac{2}{s-1}|\podpool\cap\targetlang_n|+O_s(1).    
\end{equation}

Partitioning $\advlang$ into generated strings, stolen
pod elements, and the sets $B_g,B_b$, we obtain
\begin{equation}
|\advlang\cap\targetlang_n|
\le
\left(3+\frac{6}{s}+\frac{6}{s-1}\right)
|\guesslang\cap\targetlang_n|+O_s(1).    
\end{equation}

Dividing by $n$, taking $\liminf_{n\to\infty}$, and
then letting $s\to\infty$ proves
\begin{equation}
    \setlowerrecall{\exhaustion{\targetlang}}{\guesslang}
\ge
\frac{\setlowerrecall{\exhaustion{\targetlang}}{\advlang}}{3}.
\end{equation}

Combining both bounds gives
\begin{equation}
\setlowerrecall{\exhaustion{\targetlang}}{\guesslang}\ge\ \max \Big( 1-\setupperrecall{\exhaustion{\targetlang}}{\advlang} , \frac{\setlowerrecall{\exhaustion{\targetlang}}{\advlang}}{3}\Big).
\end{equation}
Note that we cannot simply add up the bounds. This is because \citeposs{kleinberg_language_2025} routine guarantees generation from $\targetlang$, not from $\advlang$, and might hence generate solely from $\targetlang \setminus \advlang$.
\paragraph{Precision.} As all outputs are novel, $|\guesslang_n| = n$. For each $\guesslang_n$, the number of outputs from explorations (and hence possible hallucinations) is at most $\explorationidx(n)$; furthermore, \citeposs{kleinberg_language_2025} algorithm guarantees stabilization of tail precision to $1$ in finite time, meaning that all outputs from safe rounds except finitely many lie in $\targetlang$.
Hence
\begin{equation}
    |\guesslang_n\cap\targetlang| \ge n - \explorationidx(n) - \bigo(1)
\end{equation}
and therefore
\begin{align}
    \setlowerprecision{\targetlang}{\exhaustion{\guesslang}}
&= \liminf_{n\to\infty}\frac{|\guesslang_n\cap\targetlang|}{|\guesslang_n|}
\\&\ge \liminf_{n\to\infty}\left(1-\frac{\explorationidx(n) + \bigo(1)}{n}\right)
\\&= 1.
\end{align}

\end{proof}

\begin{algorithm}[t]
\caption{Language generation in the limit with novelty and without perfect tail precision. \textsc{UpdateChain} takes the adversary's input and a sequence of languages and returns an updated sequence of languages and a chain of intersections of consistent languages; \textsc{IdIntersect} takes the timestep and a level index, two intersection chains and an identified intersection and returns the next level index and the next identified intersection; \textsc{AggrSet} takes two identified intersections and two intersection chains and returns an aggressive set, i.e., a superset of the current identified intersection. \textbf{Input:} adversarial exhaustion $\{\advlang_n\}_{n= 0}^\infty$ with $\advlang_n \uparrow \advlang \subseteq \targetlang$, enumeration $\exhaustion{E}$ of $\kleene{\alphabet}$, infinite exploration set $\exploration \subset \N$, pod size sequence $(s_n)_{n\geq 1}$, ordered sequence $(\langs'_i)_{i\ge1}$ of languages in $\coll$ \textbf{Ensures:} Increasing exhaustion $\exhaustion{\guesslang}=\{\guesslang_n\}_{n= 0}^\infty$.}
\label{alg:gen-with-novelty}
\begin{algorithmic}[1]
\STATE $\guesslang_0 \leftarrow \emptyset$, $n\leftarrow 1$, $j \leftarrow 1$, $\unavailableset_0 \leftarrow \emptyset$ \COMMENT {unavailable strings}, $\podpool_0 \leftarrow \emptyset$ \COMMENT{pod pool}, $(\langs^{(0)}_i)_{i\ge1} \leftarrow (\langs'_i)_{i\ge1}$\COMMENT {consistent languages}, $k_0 \leftarrow 1$, $\advlang_{\text{prev}} \gets \emptyset$, $\intersect^{(0)}\gets\emptyset$, $\chain_0\gets\emptyset$ 

\STATE $\explorationidx_{\text{next}} \leftarrow \min(\exploration)$

\FOR{$m = 1,2,3,\dots$}
    \STATE Observe $\advlang_m$
    \IF{$m \neq \explorationidx_{\text{next}}$}
        
         \STATE $\bigl((\langs^{(n)}_i)_{i\ge1}, \chain_n\bigr) \gets$
         $\textsc{UpdateChain}\bigl(\advlang_m \setminus \advlang_{\text{prev}}, (\langs^{(n-1)}_i)_{i\ge1}\bigr)$
        \STATE $(\intersect^{(n)}, k_n) \gets$
        \textsc{IdIntersect}$(n, k_{n-1}, \intersect^{(n-1)}, \chain_{n-1}, \chain_n)$
        \STATE $\widetilde{\intersect}^{(n)} \gets$\textsc{AggrSet}$(n,\intersect^{(n-1)},\intersect^{(n)}, \chain_{n-1}, \chain_n)$
        \STATE $\unavailableset_n \gets \unavailableset_{n-1} \cup \advlang_m \cup \guesslang_{m-1}$. 
        \STATE $\Pod_n\gets$ the $s_n$ $\prec$-smallest elements of $\widetilde{\intersect}^{(n)}\setminus\unavailableset_n$
        \STATE $\unavailableset_n \gets \unavailableset_n \cup \Pod_n$
        \STATE $\podpool_n \gets \podpool_{n-1} \cup \Pod_n$.
        \STATE $\guesslang_m \gets \guesslang_{m-1} \cup \min_{\prec}(\podpool_n\setminus (\advlang_m \cup \guesslang_{m-1}))$
        \STATE $\unavailableset_n \gets \unavailableset_n \cup \guesslang_m$
        \STATE $\advlang_{\text{prev}} \gets \advlang_m$
        \STATE $n \gets n+1$ \LineComment{safe rounds}
    \ELSE
        \WHILE{$E_j \subseteq \advlang_m \cup \guesslang_{m-1}$}
            \STATE $j \gets j + 1$
        \ENDWHILE
        \STATE $\guesslang_m \gets \guesslang_{m-1} \cup \Delta E_j$
        \STATE $j \gets j + 1$
        \STATE $\explorationidx_{\text{next}} \leftarrow \min(\{\explorationidx \in \exploration \colon \explorationidx > \explorationidx_{\text{next}}\})$ \LineComment{exploration}
    \ENDIF
\ENDFOR
\end{algorithmic}
\end{algorithm}
\section{Generation with \texorpdfstring{$\gamma$}{gamma}-novelty}\label{app:gen-with-gamma-novelty}

\begin{lemma}\label{app:gamma-exploration-existence}
    There exists a $\gamma$-admissible exploration set (\cref{def:gamma-admissible-exploration}) for every $\gamma<1$.
\end{lemma}

\begin{proof}
 Let $M\ge \max\{2,\lceil 1/(1-\gamma)\rceil\}$ and set $\exploration=\{M2^i:i\in\N_0\}$. Then $\exploration$ is infinite, has no consecutive elements and satisfies $\explorationidx(m)/m\to0$. Furthermore, 
 \begin{equation}
 \explorationidx(m) \leq \frac{m}{M} \leq (1-\gamma)m.
 \end{equation}
 \end{proof}
\subsection{Generation with \texorpdfstring{$\gamma$}{gamma}-novelty without perfect tail precision}\label{app:gen-with-gamma-novelty-without-tail-prec}
\TheoremGenerationWithGammaNovelty*
\begin{proof} \Cref{alg:gen-with-gamma-novelty} combines safe rounds with exploration rounds scheduled according to a $\gamma$-admissible exploration set. Since exploration rounds are not consecutive and the safe routine is not advanced during exploration rounds, the safe routine sees a $2$-batched adversary. The exploration rounds output elements from an enumeration $\exhaustion{E}$ of $\kleene{\alphabet}$. In those rounds, we allow the generator to output strings that have been previously revealed by the adversary while it skips strings that already belong to $\guesslang_n$. Thus, $|\guesslang_n|=n$. An exploration round either does not influence the safe rounds by outputting a string that the safe rounds would never produce or makes the safe routine advance more quickly. The proof of precision is analogous to \Cref{thm:noveltybound}. It remains to show that $\gamma$-novelty is satisfied and that the recall is $1$.

\paragraph{$\gamma$-novelty.} The algorithm chooses a $\gamma$-admissible exploration set $\exploration$.
All safe rounds produce novel outputs, i.e., outputs that are chosen outside $\advlang_n \cup \guesslang_{n-1}$. Therefore, non-novel outputs can occur only on exploration rounds. Let $N_n$ be the cumulative novel set from \Cref{def:gamma-novelty}. 
For every timestep $n$, by \Cref{def:gamma-admissible-exploration}
\begin{equation}
    |\guesslang_n\setminus N_n|
\le
\explorationidx(n)
\le
(1-\gamma)n.
\end{equation}

Combined with $|\guesslang_n| = n$, it follows that
\begin{equation}
    \frac{|N_n|}{|\guesslang_n|}
=
1-\frac{|\guesslang_n\setminus N_n|}{|\guesslang_n|}
\ge \gamma.
\end{equation}
Thus, the generator is $\gamma$-novel.

\paragraph{Recall.} Let $\exhaustion{E}=\{E_j\}_{j=1}^\infty$ be the fixed enumeration of $\kleene{\alphabet}$. In each exploration round, the algorithm outputs the first element $\Delta E_j$ such that $\Delta E_j\nsubseteq\guesslang_{n-1}$. Since $\exploration$ is infinite, this exploration routine is invoked infinitely many times. We show that every element of $\kleene{\alphabet}$ is eventually generated. Fix $m\in\N$. For $m=1$, either $\Delta E_1$ is output in the first exploration round, or it has already been generated by a safe round. In either case, $\Delta E_1\in\guesslang$.

Now assume that, for some $m>1$, all elements in $E_{m^\star}$ for every $m^\star < m$ have already been added to $\guesslang$. Consider the next exploration round. When the exploration routine scans the enumeration $\exhaustion{E}$, it skips all $E_{m^\star}$ with $m^\star<m$, since they already belong to the current guessed
language. Hence, the routine reaches $E_m$. If $\Delta E_m\nsubseteq\guesslang_{n-1}$, it is output in this exploration round. If it is not output in this exploration round, then this can only be because
$\Delta E_m$ has been output in a previous safe round and thus $\Delta E_m \subseteq\guesslang_{n-1}$ already. In either case,
$\Delta E_m\subseteq\guesslang$.

By induction, every element of $\kleene{\alphabet}$ is eventually contained in $\guesslang$, and therefore $\kleene{\alphabet}\subseteq\guesslang$. Since
$\targetlang\subseteq\kleene{\alphabet}$, it follows that
$\targetlang\subseteq\guesslang$. Therefore,
\begin{equation}
    \setlowerrecall{\exhaustion{\targetlang}}{\guesslang}
    =
    \liminf_{n\to\infty}
    \frac{|\guesslang\cap\targetlang_n|}{|\targetlang_n|}
    =
    1.
\end{equation}

\end{proof}

\begin{algorithm}[t]
\caption{Language generation in the limit with $\gamma$-novelty and without perfect tail precision. \textsc{UpdateChain} takes the adversary's input and a sequence of languages and returns an updated sequence of languages and a chain of intersections of consistent languages; \textsc{IdIntersect} takes the timestep and a level index, two intersection chains and an identified intersection and returns the next level index and the next identified intersection; \textsc{AggrSet} takes two identified intersections and two intersection chains and returns an aggressive set, i.e., a superset of the current identified intersection. \textbf{Input:} adversarial exhaustion $\{\advlang_n\}_{n= 0}^\infty$ with $\advlang_n \uparrow \advlang \subseteq \targetlang$, enumeration $\exhaustion{E}$ of $\kleene{\alphabet}$, infinite $\gamma$-admissible exploration set $\exploration \subset \N$, pod size sequence $(s_n)_{n\geq 1}$, ordered sequence $(\langs'_i)_{i\ge1}$ of languages in $\coll$ \textbf{Ensures:} Increasing exhaustion $\exhaustion{\guesslang}=\{\guesslang_n\}_{n= 0}^\infty$.}
\label{alg:gen-with-gamma-novelty}
\begin{algorithmic}[1]
\STATE $\guesslang_0 \leftarrow \emptyset$, $n\leftarrow 1$, $j \leftarrow 1$, $\unavailableset_0 \leftarrow \emptyset$ \COMMENT {unavailable strings}, $\podpool_0 \leftarrow \emptyset$ \COMMENT{pod pool}, $(\langs^{(0)}_i)_{i\ge1} \leftarrow (\langs'_i)_{i\ge1}$\COMMENT {consistent languages}, $k_0 \leftarrow 1$, $\advlang_{\text{prev}} \gets \emptyset$, $\intersect^{(0)}\gets\emptyset$, $\chain_0\gets\emptyset$ 

\STATE $\explorationidx_{\text{next}} \leftarrow \min(\exploration)$

\FOR{$m = 1,2,3,\dots$}
    \STATE Observe $\advlang_m$
    \IF{$m \neq \explorationidx_{\text{next}}$}
        
         \STATE $\bigl((\langs^{(n)}_i)_{i\ge1}, \chain_n\bigr) \gets$
         $\textsc{UpdateChain}\bigl(\advlang_m \setminus \advlang_{\text{prev}}, (\langs^{(n-1)}_i)_{i\ge1}\bigr)$
        \STATE $(\intersect^{(n)}, k_n) \gets$
        \textsc{IdIntersect}$(n, k_{n-1}, \intersect^{(n-1)}, \chain_{n-1}, \chain_n)$
        \STATE $\widetilde{\intersect}^{(n)} \gets$\textsc{AggrSet}$(n,\intersect^{(n-1)},\intersect^{(n)}, \chain_{n-1}, \chain_n)$
        \STATE $\unavailableset_n \gets \unavailableset_{n-1} \cup \advlang_m \cup \guesslang_{m-1}$. 
        \STATE $\Pod_n\gets$ the $s_n$ $\prec$-smallest elements of $\widetilde{\intersect}^{(n)}\setminus\unavailableset_n$
        \STATE $\unavailableset_n \gets \unavailableset_n \cup \Pod_n$
        \STATE $\podpool_n \gets \podpool_{n-1} \cup \Pod_n$.
        \STATE $\guesslang_m \gets \guesslang_{m-1} \cup \min_{\prec}(\podpool_n\setminus (\advlang_m \cup \guesslang_{m-1}))$
        \STATE $\unavailableset_n \gets \unavailableset_n \cup \guesslang_m$
        \STATE $\advlang_{\text{prev}} \gets \advlang_m$
        \STATE $n \gets n+1$ \LineComment{safe rounds}
    \ELSE
        \WHILE{$E_j \subseteq \guesslang_{m-1}$}
            \STATE $j \gets j + 1$
        \ENDWHILE
        \STATE $\guesslang_m \gets \guesslang_{m-1} \cup \Delta E_j$
        \STATE $j \gets j + 1$
        \STATE $\explorationidx_{\text{next}} \leftarrow \min(\{\explorationidx \in \exploration \colon \explorationidx > \explorationidx_{\text{next}}\})$ \LineComment{exploration}
    \ENDIF
\ENDFOR
\end{algorithmic}
\end{algorithm}

\begin{corollary}\label{corr:full-alpha-nov-zero-tail-prec}
Fix any $\gamma\in[0,1)$. For any $\targetlang\in\coll$ and any single-step full
adversarial exhaustion $\exhaustion{\advlang}$ with
$\advlang_n\uparrow\targetlang$,
\Cref{alg:gen-with-gamma-novelty} generates a $\gamma$-novel exhaustion
$\exhaustion{\guesslang}=\{\guesslang_n\}_{n=0}^\infty$ with
$\guesslang_n\uparrow\guesslang$ such that
\begin{equation}
    \setlowerrecall{\exhaustion{\targetlang}}{\guesslang}=1,
    \qquad
    \setlowerprecision{\targetlang}{\exhaustion{\guesslang}}=1.
\end{equation}
\end{corollary}

\subsection{Generation with \texorpdfstring{$\gamma$}{gamma}-novelty and with perfect tail precision}
\label{app:gamma-novel-perfect-tail}

We consider $\gamma$-novel generation under the requirement of perfect tail precision. In contrast to the strict novelty setting with $\gamma = 1$, a $\gamma$-novel generator for $\gamma \in [0,1)$ may use a controlled fraction of non-novel outputs. We use these non-novel outputs to repeat all the strings revealed by the adversary. In \Cref{alg:gen-with-gamma-novelty-perfect-tail}, we interleave safe rounds with repetition rounds.

\begin{algorithm}[t]
\caption{Language generation in the limit with $\gamma$-novelty and perfect tail precision. \textsc{UpdateChain} takes the adversary's input and a sequence of languages and returns an updated sequence of languages and a chain of intersections of consistent languages; \textsc{IdIntersect} takes the timestep and a level index, two intersection chains and an identified intersection and returns the next level index and the next identified intersection; \textsc{AggrSet} takes two identified intersections and two intersection chains and returns an aggressive set, i.e., a superset of the current identified intersection. 
\textbf{Input:} adversarial exhaustion $\{\advlang_n\}_{n=0}^\infty$ with
$\advlang_n \uparrow \advlang \subseteq \targetlang$, novelty parameter
$\gamma\in[0,1)$, pod size sequence $(s_n)_{n\geq 1}$, ordered sequence
$(\langs'_i)_{i\ge1}$ of languages in $\coll$.
\textbf{Ensures:} increasing exhaustion
$\exhaustion{\guesslang}=\{\guesslang_n\}_{n=0}^\infty$.}
\label{alg:gen-with-gamma-novelty-perfect-tail}
\begin{algorithmic}[1]
\STATE $L \leftarrow \left\lceil \frac{1}{1-\gamma}\right\rceil$
\STATE $\guesslang_0 \leftarrow \emptyset$, $n\leftarrow 1$,
$\unavailableset_0 \leftarrow \emptyset$ \COMMENT{unavailable strings},
$\podpool_0 \leftarrow \emptyset$ \COMMENT{pod pool},
$(\langs^{(0)}_i)_{i\ge1} \leftarrow (\langs'_i)_{i\ge1}$ \COMMENT{consistent languages},
$k_0 \leftarrow 1$, $\advlang_{\text{prev}} \gets \emptyset$, $\intersect^{(0)}\gets\emptyset$, $\chain_0\gets\emptyset$ 

\FOR{$m = 1,2,3,\dots$}
    \STATE Observe $\advlang_m$

    \IF{$m \equiv 0 \pmod L$ \AND $\advlang_m\setminus\guesslang_{m-1}\neq\emptyset$}
        \STATE $\guesslang_m \gets
        \guesslang_{m-1}\cup
        \{\min_{\prec}(\advlang_m\setminus\guesslang_{m-1})\}$
        \LineComment{repetition round}
    \ELSE
        \STATE $\bigl((\langs^{(n)}_i)_{i\ge1}, \chain_n\bigr) \gets$
        $\textsc{UpdateChain}\bigl(\advlang_m \setminus \advlang_{\text{prev}},
        (\langs^{(n-1)}_i)_{i\ge1}\bigr)$
        \STATE $(\intersect^{(n)}, k_n) \gets$
        \textsc{IdIntersect}$(n, k_{n-1}, \intersect^{(n-1)}, \chain_{n-1}, \chain_n)$
        \STATE $\widetilde{\intersect}^{(n)} \gets$
        \textsc{AggrSet}$(n,\intersect^{(n-1)},\intersect^{(n)}, \chain_{n-1}, \chain_n)$
        \STATE $\unavailableset_n \gets
        \unavailableset_{n-1} \cup \advlang_m \cup \guesslang_{m-1}$
        \STATE $\Pod_n\gets$ the $s_n$ $\prec$-smallest elements of
        $\widetilde{\intersect}^{(n)}\setminus\unavailableset_n$
        \STATE $\unavailableset_n \gets \unavailableset_n \cup \Pod_n$
        \STATE $\podpool_n \gets \podpool_{n-1} \cup \Pod_n$
        \STATE $\guesslang_m \gets
        \guesslang_{m-1} \cup
        \{\min_{\prec}(\podpool_n\setminus(\advlang_m\cup\guesslang_{m-1}))\}$
        \STATE $\unavailableset_n \gets \unavailableset_n \cup \guesslang_m$
        \STATE $\advlang_{\text{prev}} \gets \advlang_m$
        \STATE $n \gets n+1$ \LineComment{safe round}
    \ENDIF
\ENDFOR
\end{algorithmic}
\end{algorithm}

\begin{theorem}\label{thm:bounds-alpha-nov-tail-prec}
    Let $\coll$ be a countable collection of languages and let $\targetlang\in\coll$ be the target language. Let $\gamma\in[0,1)$ be the required fraction of novel strings. 
Fix an enumeration $\exhaustion{\targetlang}=\{\targetlang_n\}_{n=0}^\infty$ of $\targetlang$.
There exists a $\gamma$-novel generator $\gen$ such that, for any single-step adversarial exhaustion
$\exhaustion{\advlang}=\{\advlang_n\}_{n=0}^\infty$ with $\advlang_n\uparrow \advlang\subseteq\targetlang$
and $\setlowerrecall{\exhaustion{\targetlang}}{\advlang}\ge \alpha$, the generated single-step exhaustion
$\exhaustion{\guesslang}$ satisfies:
\begin{equation}
    \setlowerrecall{\exhaustion{\targetlang}}{\guesslang} \geq \alpha,
    \qquad
    \setlowerprecision{\targetlang}{\exhaustion{\guesslang}} = 1,
    \qquad
    \setlowertailprecision{\targetlang}{\exhaustion{\guesslang}} = 1.
\end{equation}
Moreover, the recall guarantee is tight in an adversarial sense.
\end{theorem}

\begin{proof}
The case $\gamma = 0$ is covered in \Cref{thm:gen-valid-without-novelty}. Now let $\gamma \in (0,1)$ and let $L := \left\lceil \frac{1}{1-\gamma}\right\rceil$. We divide time into blocks of length $L$ and run \Cref{alg:gen-with-gamma-novelty-perfect-tail}. In each block, we run safe rounds using the pods algorithm by \citeauthor{kleinberg_language_2025} (see \cref{alg:pods-generator}) for $L-1$ rounds and take one repetition round at the end of the block. The subroutine for the pods algorithm is updated only on safe rounds, using all adversarial strings revealed since the previous safe round. Since there is at most one repetition round between two consecutive safe rounds, and the adversary is single-step bounded, the safe subroutine sees a $2$-batched adversary. On repetition rounds, the generator outputs the smallest adversary-revealed string that has not yet been generated. Formally, on a timestep $m$, a repetition round outputs the $\prec$-smallest element of $\advlang_m\setminus \guesslang_{m-1}$, if this set is nonempty. If $\advlang_m\setminus \guesslang_{m-1}=\emptyset$, the generator instead takes a safe round. The repetition outputs are always valid because $\advlang\subseteq \targetlang$. Furthermore, the repetition rounds do not interfere with the $2$-batched safe routine in any way, as they never output unused elements from pods that the safe routine might produce later.

\paragraph{$\gamma$-novelty.}
The safe rounds output only novel strings: the output is chosen outside $\advlang_m \cup \guesslang_{m-1}$, i.e., outside the strings already revealed by the adversary or previously generated. Only repetition rounds may be non-novel. Since there is at most one repetition round in each block of length $L$, and since the repetition round occurs at the end of the block, every $\guesslang_m$ contains at most one non-novel output per completed block. If $\advlang_m\setminus \guesslang_{m-1}=\emptyset$ on a repetition round, the generator takes a safe round instead, which can only increase the fraction of novel outputs.

Let $N_n$ denote the cumulative novel set from \Cref{def:gamma-novelty}. Consider the set $\guesslang_n\setminus N_n$ of non-novel strings generated up to timestep $n$. It holds that
\begin{equation}
    |\guesslang_n\setminus N_n|
    \le
    \left\lfloor \frac{|\guesslang_n|}{L}\right\rfloor
    \le
    \frac{|\guesslang_n|}{L}.
\end{equation}
Therefore, we have
\begin{equation}
\frac{|N_n|}{|\guesslang_n|}
=
1-\frac{|\guesslang_n\setminus N_n|}{|\guesslang_n|}
\ge
1-\frac{1}{L}
\ge
\gamma,
\end{equation}
where the last inequality follows from
$L\ge 1/(1-\gamma)$. Hence the generator is $\gamma$-novel.

\paragraph{Tail precision.}
All outputs from repetition rounds belong to $\advlang$, and therefore to $\targetlang$. On safe rounds, we
run the $2$-batched \citeposs{kleinberg_language_2025} algorithm. It guarantees stabilization of tail precision to $1$ in finite time, meaning that all outputs from safe rounds except finitely many lie in $\targetlang$. Hence, after some finite timestep, every generated string, whether produced by a safe round or by a repetition round, lies in $\targetlang$.
Thus
\begin{equation}
    \setlowertailprecision{\targetlang}{\exhaustion{\guesslang}}=1.
\end{equation}

\paragraph{Precision.} By \cref{lem:tail-precision-implies-precision}, a tail precision of $1$ also implies
\begin{equation}
    \setlowerprecision{\targetlang}{\exhaustion{\guesslang}} = 1.
\end{equation}

\paragraph{Recall.}
We show that every string revealed by the adversary is eventually generated.
Let $\str\in \advlang$. There exists a finite timestep $m^\star$ such
that $\str\in \advlang_m$ for all $m\ge m^\star$. From that time onward, unless $\str$ has
already been generated by a safe round, it remains in $\advlang_m\setminus \guesslang_{m-1}$. The repetition rounds continue to occur as long as $\advlang_m \setminus \guesslang_{m-1} \neq \emptyset$. Moreover, since $\prec$ is induced by an enumeration of $\kleene{\alphabet}$, each string has only finitely many $\prec$-predecessors. Hence, there are only finitely many elements that can be output in repetition rounds before outputting $\str$, meaning that $\str$ is generated eventually. Hence, every $\str \in \advlang$ is output eventually, which yields 
\begin{equation}
    \setlowerrecall{\exhaustion{\targetlang}}{\guesslang} = \liminf_{n\to\infty}
\frac{|\targetlang_n\cap \guesslang|}{|\targetlang_n|}
\ge
\liminf_{n\to\infty}
\frac{|\targetlang_n\cap \advlang|}{|\targetlang_n|}
= \setlowerrecall{\exhaustion{\targetlang}}{\advlang}
\ge
\alpha.
\end{equation}
The tightness of the recall bound in an adversarial sense follows from the same argument as in \Cref{thm:gen-valid-without-novelty}.
\end{proof}

The following corollary describes the special case of a full exhaustion $\exhaustion{\advlang}=\{\advlang_n\}_{n=0}^\infty$ with $\advlang_n \uparrow \targetlang$, meaning $\setlowerrecall{\exhaustion{\targetlang}}{\advlang} = 1$.
\begin{corollary}\label{corr:bounds-alpha-nov-perf-tail-prec}

Let $\coll$ be a countable collection of languages and let $\targetlang\in\coll$ be the target language. Let $\gamma\in[0,1)$ be the required fraction of novel strings. 
Fix an enumeration $\exhaustion{\targetlang}=\{\targetlang_n\}_{n=0}^\infty$ of $\targetlang$.
There exists a $\gamma$-novel generator $\gen$ such that, for every single-step full adversarial exhaustion
$\exhaustion{\advlang}=\{\advlang_n\}_{n=0}^\infty$ with $\advlang_n\uparrow \targetlang$, the generated single-step exhaustion
$\exhaustion{\guesslang}$ satisfies:
\begin{equation}
    \setlowerrecall{\exhaustion{\targetlang}}{\guesslang} = 1,
    \qquad
    \setlowerprecision{\targetlang}{\exhaustion{\guesslang}} = 1,
    \qquad
    \setlowertailprecision{\targetlang}{\exhaustion{\guesslang}} = 1.
\end{equation}
\end{corollary}

\end{document}